%% file: draft1.tex
\pdfoutput=1
\documentclass[10pt,journal,letterpaper,compsoc]{IEEEtran}

\usepackage[footnotesize]{caption}
\usepackage{xspace,amssymb,amsmath,amsthm,graphicx,srcltx,dblfloatfix,url}

\makeatletter
\def\normalsize{\@setfontsize{\normalsize}{9.5bp}{12.00pt}}
\normalsize
\makeatother

\usepackage[nocompress]{cite}

\usepackage{amsmath,amscd}
\usepackage{graphicx}
\graphicspath{{./}{./Fig/}{./Figs/}}

\usepackage{algorithm2e}
\let\savedalgorithm\algorithm
\let\savedendalgorithm\endalgorithm

\usepackage{algorithm}

\theoremstyle{plain}

\newtheorem{theorem}{Theorem}[section]

\newtheorem{proposition}{Proposition}[section]

\renewcommand{\u}{\mathbf u}
\renewcommand{\v}{\mathbf v}

\newcommand{\x}{\mathbf x}
\newcommand{\w}{\mathbf w}

\newcommand{\z}{\mathbf z}

\newcommand{\e}{\boldsymbol e}

\newcommand{\A}{\mathbf A}
\newcommand{\C}{\mathbf C}
\newcommand{\X}{\mathbf X}
\renewcommand{\H}{\mathbf H}

\newcommand{\bmu}{{\boldsymbol \mu}}

\newcommand{\bSigma}{ {\boldsymbol \Sigma } }
\newcommand{\brho}{ {\boldsymbol \rho} }

\newcommand{\Q}{\mathbf Q}
\newcommand{\q}{\mathbf q}

\newcommand{\etal}{{\em et al.}\xspace}
\newcommand{\ie}{{\em i.e.}\xspace}

\newcommand{\st}{{\rm s.t.}\xspace}

\newcommand{\comment}[1]{}

\def\psd{\succcurlyeq}  
\def\nsd{\preccurlyeq}

\newcommand{\cH}{\mathcal H}
\newcommand{\cX}{\mathcal X}

\def\T{{\!\top}}
\def\Real{\mathbb{R}}
 
\newcommand{\norm}[2][2]{\ensuremath{ \left\| #2 \right\|_{ \mathrm{#1} } } }

\newcommand{\ADot}{ \ensuremath{  - \,} }

\def\NICTAFunding{{NICTA is funded by the 
Australian Government as represented by the Department of
Broadband, Communications and the Digital Economy and the 
Australian Research Council through the
ICT Center of Excellence program.}\xspace
}

\begin{document}

\title{
      Optimally Training a Cascade Classifier
}

\author{
   Chunhua Shen,
   Peng Wang,
   and Anton van den Hengel
 \IEEEcompsocitemizethanks{
 \IEEEcompsocthanksitem
  C.~Shen is with NICTA, Canberra Research Laboratory, ACT 2601, Australia;
  and Australian National University, Canberra, ACT 0200, Australia. 
  \NICTAFunding.
  C. Shen's research is in part supported by 
  the Australian Research Council through the  
  Research in Bionic Vision Science and Technology Initiative. 
  E-mail: chunhua.shen@nicta.com.au.
 \IEEEcompsocthanksitem
  P. Wang is with Beihang University, Beijing 100191, China.
  E-mail: peng.wang@nicta.com.au. His contribution was made while
  visiting NICTA, Canberra Research Laboratory, and Australian
  National University.
\IEEEcompsocthanksitem
  A. van den Hengel is with University
  of Adelaide, Adelaide, SA 5005, Australia. 
  A. van den Hengel's participation in this research
  was supported under the Australian Research Council's Discovery Projects funding scheme 
  (project number DP0988439).
  E-mail: anton.vandenhengel@adelaide.edu.au.
 }
}

\markboth{Shen et al.: Optimally Training a Cascade Classifier}{}

\IEEEcompsoctitleabstractindextext{%

\begin{abstract}

         Cascade classifiers are widely used in real-time object detection.
         Different from conventional classifiers that are designed for a low overall classification
         error rate,     a classifier in each node of the cascade is required to achieve an
         extremely high detection rate and moderate false positive rate.  Although there are a few
         reported methods addressing  this requirement in the context of object detection, there is
         no a principled feature selection method that explicitly takes into account this asymmetric
         node learning objective.  We provide such an algorithm here.  We show a special case of the
         biased minimax probability machine has the same formulation as the linear asymmetric
         classifier (LAC) of  \cite{wu2005linear}. We then  design a new boosting algorithm that
         directly optimizes the cost function of LAC. The resulting totally-corrective boosting
         algorithm is implemented by the column generation technique in convex optimization.
         Experimental results on object detection verify the effectiveness of the proposed boosting
         algorithm as a node classifier in cascade object detection, and show performance better
         than that of the current state-of-the-art.

       %

\end{abstract}

}

\maketitle

\thispagestyle{empty}

\section{Introduction}
\label{sec:intro}

        \IEEEPARstart{R}{eal-time} object detection inherently involves searching a large number of candidate image
        regions for a small number of objects.  Processing a single image, for example, can require
        the interrogation of well over a million scanned windows in order to uncover a single
        correct detection.  This imbalance in the data has an impact on the way that detectors are
        applied, but also on the training process.  This impact is reflected in the need to identify
        discriminative features from within a large over-complete feature set.

        Cascade classifiers have been proposed as a potential solution to the problem of imbalance
        in the data \cite{wu2008fast,brubaker2008,bi2006Comp,dundar2007,viola2004robust}, and have
        received significant attention due to their speed and accuracy. In this work, we propose a
        principled method by which to train a {\em boosting}-based cascade of classifiers.

        The boosting-based cascade approach to object detection was introduced by Viola and
        Jones \cite{viola2004robust,Viola2002Fast}, and has received significant subsequent
        attention
        \cite{pham07,pham08multi,Li2004Float,Paisitkriangkrai2009CVPR,shen2008face,paisitkriangkrai2008fast}.
        It also
        underpins the current state-of-the-art \cite{wu2005linear,wu2008fast}.

        The Viola and Jones approach uses a cascade of increasingly complex classifiers, each of
        which aims to achieve the best possible classification accuracy while achieving an extremely
        low false negative rate.  These classifiers can be seen as forming the nodes of a degenerate
        binary tree (see Fig.~\ref{fig:1A}) whereby a negative result from any single such node
        classifier terminates the interrogation of the current patch.  Viola and Jones use AdaBoost
        to train each node classifier in order to achieve the best possible classification accuracy.
        A low false negative rate is achieved by subsequently adjusting the decision threshold until
        the desired false negative rate is achieved.  This process cannot be guaranteed to produce
        the best classification accuracy for a given false negative rate.

        Under the assumption that each node of the cascade classifier makes independent
        classification errors, the detection rate and false positive rate of the entire cascade are:
        $ F_{\rm dr} = \prod_{ t =1}^N d_t    $  and $ F_{\rm fp} = \prod_{ t =1}^N f_t    $,
        respectively, where $d_t$ represents the detection rate of classifier $t$, $f_t$ the
        corresponding false positive rate and $N$ the number of nodes.  As pointed out in
        \cite{viola2004robust,wu2005linear}, these two equations suggest a {\em node learning
        objective}: Each node should have an extremely high detection rate $d_t $ ({\em e.g.},
        $99.7\%$) and a moderate false positive rate $ f_t $ ({\em e.g.}, $50\%$).  With the above
        values of $  d_t $ and $ f_t $, and a cascade of $ N = 20 $ nodes, then $ F_{\rm dr} \approx
        94\%$ and $ F_{\rm fp} \approx  10^ {-6} $, which is a typical design goal.

        One drawback of the standard AdaBoost approach to boosting 
        is that it does not take advantage of the cascade classifier's special structure. 
        AdaBoost only minimizes the overall classification error and does not minimize the number of
        false negatives.  In this sense, the features selected are not optimal for the purpose of
        rejecting as many negative examples as possible.
        Viola and Jones proposed a solution to this problem in AsymBoost \cite{Viola2002Fast}
        (and its variants \cite{pham07,pham08multi,wang2010asym,masnadi2007asym}) by
        modifying the exponential loss function so as to more greatly penalize false negatives.
        AsymBoost achieves better detection rates than AdaBoost, but still addresses the node
        learning goal \emph{indirectly}, and cannot be guaranteed to achieve the optimal solution.

        Wu {\em et al.} explicitly studied the node learning goal and proposed 
        to use linear asymmetric classifier (LAC) and Fisher linear discriminant analysis (LDA)
        to adjust the weights on a set of features selected by AdaBoost or AsymBoost 
        \cite{wu2005linear,wu2008fast}.
        Their experiments indicated that with this post-processing technique 
        the node learning objective can be better met, which is translated into improved 
        detection rates. 
        In Viola and Jones' framework, boosting is used to select features and at the same time to
        train a strong classifier. Wu {\em et al.}'s work separates these two tasks:
        AdaBoost or AsymBoost is used to select features; and as a second step, LAC or LDA is used to
        construct a strong classifier by adjusting the weights of the selected features. 
        The node learning objective is only considered at the 
        second step. At the first step---feature selection---the node learning objective is
        not explicitly considered at all. We conjecture  that 
        {\em further improvement may be gained
        if the node learning objective is explicitly 
        taken into account at both steps}. 
        We thus propose new  boosting algorithms to implement this idea and verify this conjecture.
        A preliminary version of this work was published in
        Shen {\em et al.} \cite{shen2010eccv}. 

        Our major contributions are as follows.  
        \begin{enumerate}
           \item
                Starting from the theory of minimax probability machines (MPMs),
                we derive a simplified 
                version of the biased minimax probability machine, which has the same formulation as 
                the linear asymmetric classifier of \cite{wu2005linear}. 
                We thus show the underlying connection between MPM and LAC. 
                Importantly, this new interpretation weakens some of the restrictions 
                on the acceptable input data distribution imposed by LAC. 
           \item
                We develop new boosting-like algorithms by directly
                minimizing the objective function of the
                linear asymmetric classifier, which results in an algorithm that we label LACBoost. 
                We also propose FisherBoost on the basis of Fished LDA rather than LAC.  Both
                methods may be used to identify the feature set that optimally achieves the node
                learning goal when training a cascade classifier.  To our knowledge, this is the
                first attempt to design such a feature selection method. 
           \item

                LACBoost and FisherBoost share similarities with LPBoost 
                \cite{Demiriz2002LPBoost} in the sense that both use
                column generation---a technique originally proposed
                for large-scale linear programming (LP). 
                Typically, the Lagrange dual problem 
                is solved at each iteration in column generation. We instead solve
                the primal quadratic programming (QP) problem, which has a special structure
                and entropic gradient (EG)
                can be used to solve the problem very efficiently. 
                Compared with general interior-point based QP solvers, EG is much faster. 
           \item
                 We apply LACBoost and FisherBoost to object detection and better performances are
                 observed over the state-of-the-art methods \cite{wu2005linear,wu2008fast,maji2008}.
                 The results confirm our conjecture and show the effectiveness of LACBoost and
                 FisherBoost.  These methods can be immediately applied to other asymmetric
                 classification problems.
         \end{enumerate}

                Moreover, we analyze the condition that makes the validity of LAC,
                and show that the multi-exit cascade might be more suitable
                for applying LAC learning of \cite{wu2005linear,wu2008fast} (and our LACBoost)
                rather than Viola-Jones standard cascade.
         
                As observed in Wu {\em et al.} \cite{wu2008fast}, in many cases, LDA even performs
                better than LAC. In our experiments, we have also observed similar phenomena. 
                Paisitkriangkrai {\em et al.} \cite{Paisitkriangkrai2009CVPR} empirically showed
                that LDA's criterion can be used to achieve better detection results. 
                An explanation of why LDA works so well for object detection
                is missing in the literature. 
                Here we  demonstrate that in the context of
                object detection, LDA can be seen as a regularized version of LAC in approximation.

         The proposed LACBoost/FisherBoost algorithm differs from traditional boosting
         algorithms in that it does not minimize a loss function.
         This opens new possibilities for designing new boosting algorithms for special purposes. 
         We have also extended column generation for optimizing nonlinear 
         optimization problems.

         \begin{figure}[t]
             \begin{center}
                 \includegraphics[width=0.45\textwidth]{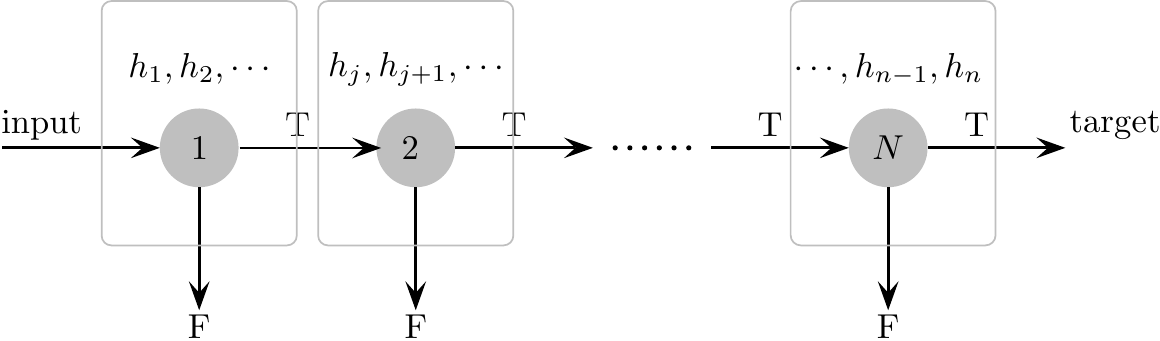}
                 \includegraphics[width=0.45\textwidth]{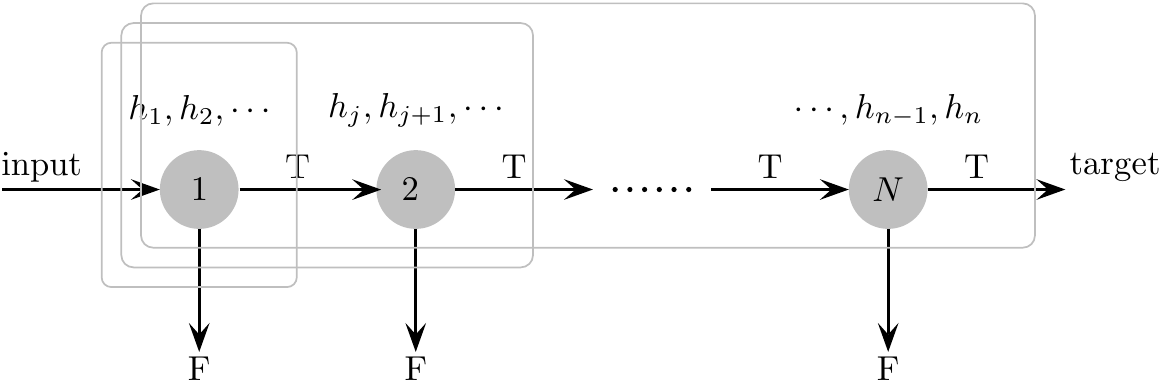}
             \end{center}
             \caption{Cascade classifiers. 
             The first one is the standard cascade of Viola and Jones \cite{viola2004robust}. 
             The second one is the multi-exit cascade proposed in \cite{pham08multi}. 
             Only those classified as true detection by all nodes will be true targets.}
             \label{fig:1A}
         \end{figure}

        \subsection{Related Work}
        
        %
        %
        %
        The three components  making up the Viola and Jones' detection approach are:
        \begin{enumerate}
        \item
              The cascade classifier, which efficiently filters out negative patches
              in early nodes while maintaining a very high detection rate;
        \item
            AdaBoost that selects informative features and at the same time trains 
            a strong classifier;
        \item
            The use of integral images, which makes the computation of Haar features 
            extremely fast.
        \end{enumerate}
        This approach has received significant subsequent attention. 
        A number of alternative cascades have been developed including the
        soft cascade \cite{Bourdev05SoftCascade}, the dynamic cascade 
        \cite{Rong2007}, and the multi-exit cascade \cite{pham08multi}. 
       %
        In this work we have adopted the multi-exit cascade that aims to improve classification
        performance by using the results of all of the weak classifiers applied to a patch so far in
        reaching a decision at each node of the tree (see Fig.~\ref{fig:1A}).
        Thus the  $ n $-th node classifier uses the results of the weak classifiers associated with
        node $n$, but also those associated with the previous $n-1$ node classifiers in the cascade.
        We show below that LAC post-processing can enhance the multi-exit cascade, and that
        the multi-exit cascade more accurately fulfills the LAC requirement that the margin
        be drawn from a Gaussian distribution.

        There have also been a number of improvements
        suggested to the Viola and Jones approach to
        the learning algorithm for constructing a classifier. 
        Wu {\em et al.}, for example, use fast forward feature selection to 
        accelerate the training procedure
        \cite{wu2003rare}. 
        Wu {\em et al.} \cite{wu2005linear} also showed that LAC may be used to deliver 
        better classification performance.         
        Pham and Cham recently proposed online asymmetric boosting 
        that considerably reduces the training time required \cite{pham07}.
        By exploiting the feature
        statistics, Pham and Cham have also designed a fast method to train weak classifiers
        \cite{Pham2007Fast}. 
        Li {\em et al.} proposed FloatBoost, which  discards
        redundant weak classifiers during AdaBoost's 
        greedy selection procedure \cite{Li2004Float}. 
        Liu and Shum also proposed KLBoost, aiming to select features that maximize the
        projected Kullback-Leibler divergence and select feature weights by minimizing the
        classification error \cite{Liu2003KL}.
        Promising results have also been reported by
        LogitBoost \cite{Tuzel2008PAMI} that employs the logistic regression 
        loss, and GentleBoost \cite{Torralba2007} that uses adaptive 
        Newton steps to fit the additive model.
        Multi-instance boosting has been introduced to object detection 
        \cite{viola2005mil,dollar08mcl,lin2009mil},
        which does not require exactly labeled locations of the targets in 
        training data.

        New features have also been designed for improving the detection
        performance. Viola and Jones' 
        Haar features are not sufficiently discriminative for detecting
        more complex objects like pedestrians, or multi-view faces. 
        Covariance features \cite{Tuzel2008PAMI} and histogram of oriented gradients (HOG)
        \cite{Dalal2005HOG} have been proposed in this context,
        and efficient implementation approaches (along the lines of integral images) are 
        developed for each. Shape context, which can also exploit integral images 
        \cite{aldavert2010integral}, was applied to human detection in thermal images
        \cite{wang2010thermal}.        
        The local binary pattern (LBP) descriptor and its variants have  been shown
        promising performance on human detection \cite{mu2008lbp,zheng2010cslbp}. 
        Recently, effort has been spent on combining complementary features, including:
        simple concatenation of HOG and LBP  \cite{wang2009hoglbp}, combination of 
        heterogeneous local features in a boosted cascade 
        classifier \cite{wu2008tradeoff},
        and Bayesian integration of intensity, depth and motion features 
        in a mixture-of-experts model \cite{enzweiler2010multi}.

        The paper is organized as follows. We briefly review the concept of
        minimax probability machine and derive the new simplified version of
        biased minimax probability machine in Section \ref{sec:mpm}. 
        Linear asymmetric classification and its connection 
        to the minimax probability machine is discussed in Section \ref{sec:LAC}. 
        In Section \ref{sec:LACBoost}, 
        we show how to design new boosting algorithms (LACBoost and FisherBoost)
        by rewriting the optimization
        formulations of LAC and Fisher LDA. 
        The new boosting algorithms are applied to object detection
        in Section \ref{sec:exp} and we 
        conclude the paper in Section \ref{sec:con}.

        \subsection{Notation}

        The following notation is used. 
        A matrix is denoted by a bold upper-case
        letter ($\X$); a column vector is denoted by a bold lower-case
        letter ($ \x $). 
        The $ i$th row of $\X $ is denoted by $ \X_{i:} $ 
        and the $ i $th column $ \X_{:i}$.
        The identity matrix is $ \bf I $ and its size should be clear
        from the context. $ \bf 1 $ and  
        $ \bf 0 $ are column vectors of $ 1$'s and $ 0$'s,
        respectively.  
        We use $ \psd, \nsd $ to denote component-wise inequalities.

        Let $ {\cal T} = \{ (\x_i, y_i  ) \}_{i = 1, \cdots, m}$ be the set of
        training data, where $ \x_i \in \cX$ and $ y_i \in \{-1,+1\}
        $, $ \forall i$. 
        The training set consists of $ m_1 $ positive training points
        and $ m_2 $ negative ones; $ m_1 + m_2 = m $. 
        Let $ h ( \cdot  ) \in \cH $ be a weak
        classifier that projects an input vector $ \x $ into 
        $\{-1, +1 \}$. 
        Note that here we consider only classifiers with discrete outputs
        although the developed methods can be applied to real-valued 
        weak classifiers too. 
        We assume that $ \cH $,
        the set from which $ h ( \cdot  ) $ is selected, is finite and has
        $n$ elements.
        %
        %

        Define the matrix $ \H^{\cal Z} \in \Real^{ m \times n }$ such that the  $ (i,j)$ entry $
        \H^{\cal Z}_{ij} =  h_j ( \x_i ) $ is the label predicted by weak classifier $ h_j(\cdot) $ for
        the datum $ \x_i $, where $\x_i$ the $i$th element of the set ${\cal Z}$. 
        In order to simplify
        the notation we eliminate the superscript when ${\cal Z}$ is the training set,
        so  $\H^{\cal Z} = \H$.
        Therefore, each
        column $ \H_{ :j }  $ of the matrix $ \H $ consists of the
        output of weak classifier $ h_j(\cdot) $ on all the training
        data; while each row $ \H_{ i: } $ contains  the outputs of all
        weak classifiers on the training datum $ \x_i $.
        Define similarly the matrix $ \A \in  \Real^{ m \times n }$ such that 
        $ \A_{ij} = y_i h_j (  \x_i ) $.
%
        Note that boosting algorithms entirely depends on the matrix $  \A  $ and
        do not directly interact with the training examples. 
        Our following discussion will thus largely focus on the matrix $\A$.
        We write the vector obtained by multiplying a matrix $  \A  $ 
        with a vector 
        $ \w $ as $ \A \w $ and its $i$th entry as $ (\A \w)_i$.
        If we let $\w$ represent the coefficients of a selected weak
        classifier then the margin of the training datum $ \x_i $ is
        $ \rho_i = \A_{ i :} \w = (\A \w)_i $ 
        and the vector of such margins for all of the training data is $\brho = \A \w$.

\section{Minimax Probability Machines}
\label{sec:mpm}

    Before we introduce our boosting algorithm, let us briefly 
    review the concept of minimax probability machines (MPM) 
    \cite{lanckriet2002mpm} first.

\subsection{Minimax Probability Classifiers}

   %
    Let $ \x_1  \in \Real^n $ and  $ \x_2 \in \Real^n  $
    denote two random vectors
    drawn from two distributions with means and covariances
    $ ( \bmu_1, \bSigma_1 ) $
    and $ (  \bmu_2, \bSigma_2  ) $, respectively. 
    Here $ \bmu_1, \bmu_2 \in \Real^n  $ and 
    $ \bSigma_1, \bSigma_2 \in \Real^{ n \times n } $. 
    We define the class labels of $ \x_1 $ and
    $ \x_2 $  as $ +1 $ and $ -1 $, w.l.o.g.  
    The minimax probability machine (MPM)
    seeks a robust separation hyperplane that can separate the two classes of data with
    the maximal probability. The hyperplane can be expressed as $ \w ^\T \x = b  $ with 
    $ \w \in \Real^n \backslash \{\bf 0 \} $ and $ b \in \Real $. 
    The problem of identifying the optimal hyperplane may then
    be formulated as 
    \begin{align}
        \label{eq:1}
        \max_{\w,b,\gamma} \,\, \gamma \,\,\,\, \st \,   
        & \left[  \inf_{ \x_1 \sim ( \bmu_1, \bSigma_1 )  } \Pr \{ \w^\T \x_1 \geq b \}  \right]
                  \geq \gamma,
        \\
        & \left[  \inf_{ \x_2 \sim ( \bmu_2, \bSigma_2 )  } \Pr \{ \w^\T \x_2 \leq b \}  \right]
                  \geq \gamma.    \notag       
    \end{align}
    Here $ \gamma $ is the lower bound of the classification accuracy
    (or the worst-case accuracy) on test data.
    This problem can be transformed into a convex problem,
    more specifically a second-order cone
    program (SOCP) \cite{boyd2004convex} and thus can be solved efficiently
    \cite{lanckriet2002mpm}.

    Before we present our results, we introduce an important proposition from 
    \cite{yu2009general}. Note that we have used different notation.
    \begin{proposition}
        For a few different distribution families, the worst-case constraint 
        \begin{equation}
            \label{eq:3}
             \left[  \inf_{ \x\sim ( \bmu, \bSigma )  } \Pr \{ \w^\T \x\leq b \}  \right]
                  \geq \gamma,
        \end{equation}
        can be written as:
        \begin{enumerate}
        \item
            if $ \x\sim (\bmu, \bSigma ) $, {\em i.e.}, 
            $ \x$ follows an arbitrary distribution with 
        mean $ \bmu $ and covariance $ \bSigma $, then
        \begin{equation}
            \label{eq:5A}
                    b \geq 
                        \w ^\T \bmu + \sqrt{ \tfrac{ \gamma }{ 1 - \gamma } }
                        \cdot 
                        \sqrt{  \w^\T \bSigma  \w }; 
        \end{equation}
        \item
            if $ \x\sim (\bmu, \bSigma )_{\rm S},$\footnote{Here
            $(\bmu, \bSigma )_{\rm S}$ denotes
            the family of distributions in $ ( \bmu, \bSigma  ) $ that are also symmetric 
            about the mean $ \bmu $. 
            $(\bmu, \bSigma )_{\rm SU}$ denotes
            the family of distributions in $ ( \bmu, \bSigma  ) $ that are additionally symmetric 
            and linear unimodal about  $ \bmu $.
            }
            then we have
        \begin{equation}
            \label{eq:5B}
            \begin{cases}
                    b \geq 
                        \w ^\T \bmu + \sqrt{ \tfrac{ 1 }{ 2 (1 - \gamma) } }
                        \cdot 
                        \sqrt{  \w^\T \bSigma  \w }, 
                             & \text{if~} \gamma \in (0.5,1);  
                    \\
                        b \geq 
                        \w ^\T \bmu, 
                             & \text{if~} \gamma \in (0,0.5];
            \end{cases}
        \end{equation}
    \item
        if
        $ \x\sim (\bmu, \bSigma )_{\rm SU} $,
%
        then 
        \begin{equation}
            \label{eq:5C}
            \begin{cases}
                    b \geq 
                        \w ^\T \bmu + \frac{2}{3}  \sqrt{ \tfrac{ 1 }{ 2 (1 - \gamma) } }
                        \cdot 
                        \sqrt{  \w^\T \bSigma  \w }, 
                             & \text{if~} \gamma \in (0.5,1);  
                    \\
                        b \geq 
                        \w ^\T \bmu, 
                             & \text{if~} \gamma \in (0,0.5];
            \end{cases}
        \end{equation}
    \item
        if $ \x$ follows a Gaussian distribution with
        mean $ \bmu $ and covariance  $ \bSigma $, \ie, $ \x\sim {\cal G}( \bmu, \bSigma ) $,
        then
        \begin{equation}
            \label{eq:5D}
                    b \geq 
                    \w ^\T \bmu + \Phi^{-1} ( \gamma ) 
                    \cdot   \sqrt{  \w^\T \bSigma  \w }, 
        \end{equation}
        where $ \Phi(\cdot)$ is the cumulative distribution function (c.d.f.) of the
        standard normal distribution $ {\cal G} (0,1)$, and $ \Phi ^{-1} (\cdot)$
        is the inverse function of $ \Phi(\cdot)$. 
        
        Two useful observations about $ \Phi ^{-1} (\cdot)$ are: 
        $ \Phi ^{-1} ( 0.5) = 0 $; 
        and $ \Phi ^{-1} ( \cdot ) $ is a monotonically increasing 
        function in its domain.  
\end{enumerate}
\label{prop:1}
    \end{proposition}
    We omit the proof here and refer the reader to \cite{yu2009general} for details.

\subsection{Biased Minimax Probability Machines}

    The formulation \eqref{eq:1} assumes that the classification problem is balanced.
    It attempts to achieve a high recognition accuracy, which assumes that
    the losses associated with all mis-classifications are identical. 
    However, in many applications this is not the case.

    Huang \etal \cite{huang2004mpm} proposed a biased version of MPM
    through a slight modification of \eqref{eq:1}, which
    may be formulated as 
    \begin{align}
        \label{eq:2}
        \max_{\w,b,\gamma} \,\, \gamma \,\,\,\, \st \,   
        & \left[  \inf_{ \x_1 \sim ( \bmu_1, \Sigma_1 )  } \Pr \{ \w^\T \x_1 \geq b \}  \right]
                  \geq \gamma,
        \\
         & \left[  \inf_{ \x_2 \sim ( \bmu_2, \Sigma_2 )  } \Pr \{ \w^\T \x_2 \leq b \}  \right]
                  \geq \gamma_\circ.    \notag       
    \end{align}
    Here $ \gamma_\circ \in (0,1) $ is a prescribed constant,
    which is the acceptable classification accuracy for the less
    important class.
    The resulting decision hyperplane 
    prioritizes the classification of
    the important class $ \x_1 $ 
    over that of the less important class $ \x_2 $.
    Biased MPM is thus expected to perform better in 
    biased classification applications.

    Huang \etal showed that \eqref{eq:2} can be iteratively
    solved via solving a sequence of SOCPs using 
    the fractional programming (FP) technique.
    Clearly it is
    significantly more computationally demanding
    to solve \eqref{eq:2} than \eqref{eq:1}.

    In this paper we are interested in the special 
    case of $ \gamma_\circ = 0.5 $ due to its important
    application in cascade object detection
    \cite{viola2004robust,wu2005linear}.
    In the following discussion, for simplicity,
    we only consider $ \gamma_\circ = 0.5 $ although some algorithms
    developed may also apply to $  \gamma_\circ < 0.5 $.

    Next we show how to re-formulate \eqref{eq:2} into a simpler quadratic program (QP) 
    based on the recent  theoretical results in \cite{yu2009general}.

    \subsection{Simplified Biased Minimax Probability Machines}

   %
    Equation \eqref{eq:5A} represents the most general  of the four cases presented in
    equations~\eqref{eq:5A} through~\eqref{eq:5D}, and is used in MPM \cite{lanckriet2002mpm}
    and the biased MPM \cite{huang2004mpm} because it does not impose constraints upon the
    distributions of $\x_1$ and $\x_2$.
    On the other hand, 
    one may take advantage of prior knowledge whenever available. For example,
    it is shown in \cite{wu2005linear} that in face detection, 
    the weak classifier outputs 
    can be well approximated by Gaussian distributions. 
    Equation \eqref{eq:5A} does not utilize any this type of {\it a priori}
    information, and hence, for many problems, \eqref{eq:5A} is too
    {\em conservative}.

    Let us consider the special case of $ \gamma = 0.5 $. 
    It is easy to see that the worst-case constraint \eqref{eq:3}
    becomes a simple linear constraint for symmetric, symmetric unimodal,
    as well as Gaussian distributions.  
    As pointed in \cite{yu2009general}, such a result is the immediate 
    consequence of symmetry because the worst-case distributions are forced to
    put probability mass arbitrarily far away on both sides of the mean. 
    In such a case any information about the covariance is neglected.

    We now apply this result into  biased MPM as represented by \eqref{eq:2}.
    Our main result is the following theorem. 
    \begin{theorem}
        With $ \gamma_\circ = 0.5 $, the biased minimax problem \eqref{eq:2} 
        can be formulated as an unconstrained problem \eqref{eq:11d}
        under the assumption
        that $ \x_2 $ follows a symmetric distribution.
        The worst-case classification accuracy for the first class, 
        $ \gamma^\star$, is obtained by solving
        \begin{equation}
            \label{eq:opt}
            \varphi( \gamma^\star )
            = \frac{ -b^\star + {a^\star } ^ \T \bmu_1 }
                   { 
                     \sqrt{ {\w^\star} ^\T \Sigma_1 \w^\star }
                   },
        \end{equation}
        where  $ \varphi(\cdot) $ is defined in \eqref{eq:6}; 
        $\{ \w^\star , b^\star \}$ is the optimal solution of \eqref{EQ:11c}
        and \eqref{eq:11d}. 
    \label{thm:2}
    \end{theorem}
    \begin{proof}
        The second constraint of \eqref{eq:2} is simply 
        \begin{equation}
            b \geq \w^\T \bmu_2.
            \label{eq:con1}
        \end{equation}
        The first constraint of \eqref{eq:2} can be handled by writing 
        $ \w^\T \x_1 \geq b $ as $ -  \w^\T \x_1 \leq - b $ 
        and applying the results 
        in Proposition \ref{prop:1}. It can be written as 
        \begin{equation}
            \label{eq:con2}
            - b + \w^\T \bmu_1 \geq \varphi( \gamma ) 
                     \sqrt{ \w^\T \Sigma_1 \w }, 
        \end{equation}
        with
        \begin{equation}
            \label{eq:6}
            \varphi( \gamma ) = 
            \begin{cases}
                \sqrt{ \frac{ \gamma }{ 1 - \gamma } }           & \text{ ~ if~} 
                            \x_1 \sim ( \bmu_1, \Sigma_1 ),
                \\
                \sqrt{ \frac{1}{ 2 (1 - \gamma) }  }               & \text{ ~ if~} 
                            \x_1 \sim ( \bmu_1, \Sigma_1 )_{ \rm S},
                \\
                \frac{2}{3}\sqrt{ \frac{1}{ 2 (1 - \gamma) }  }    & \text{ ~ if~} 
                            \x_1 \sim ( \bmu_1, \Sigma_1 )_{ \rm SU},
                \\
                \Phi^{-1} ( \gamma )        & \text{ ~ if~} 
                                            \x_1 \sim {\cal G}( \bmu_1, \Sigma_1 ).
            \end{cases}
        \end{equation}
        Let us assume that $ \Sigma_1 $ is strictly positive definite 
        (if it is only positive semidefinite, we can always add a small 
        regularization to its diagonal components).
        From \eqref{eq:con2} we have
        \begin{equation}
        \varphi( \gamma ) \leq \frac{ -b + \w^\T \bmu_1 } { \sqrt{ \w^\T \Sigma_1 \w }}.
        \label{eq:11}
        \end{equation}
        So the optimization problem becomes 
        \begin{equation}
        \max_{\w, b, \gamma } \, \gamma, \,\, 
                \st \,\, \eqref{eq:con1} \text{~and~}
        \eqref{eq:11}.
        \end{equation}

        The maximum value of $\gamma$ (which we label $ \gamma^\star $)
        is achieved when  \eqref{eq:11}
        is strictly an equality. 
        To illustrate this point,
        let  us assume that the maximum is achieved when 
        \[
        \varphi( \gamma^\star ) < \frac{ -b + \w^\T \bmu_1 } { \sqrt{ \w^\T \Sigma_1 \w }}.
        \] 
        Then a new solution can be obtained
        by increasing $  \gamma^\star $ with a positive value such that 
        \eqref{eq:11} becomes an equality. Notice that the constraint \eqref{eq:con1}
        will not be affected,
        and the new solution will be better than the previous one.
        Hence, at the optimum, \eqref{eq:opt} must be fulfilled.

        Because $ \varphi (\gamma)  $ is monotonically increasing 
        {\it for all the four cases}
        in its domain $ (0,1) $ (see Fig.~\ref{fig:1}),
        maximizing $ \gamma $ is equivalent to maximizing 
        $\varphi (\gamma)  $ 
        and this results in 
        \begin{equation}
            \label{eq:11b}
            \max_{\w, b } \,
            \frac{ -b + \w^\T \bmu_1 } 
                { \sqrt{ \w^\T \Sigma_1 \w } }, \,\, 
                \st \,
                b \geq \w^\T \bmu_2.    
        \end{equation}
        As in \cite{lanckriet2002mpm,huang2004mpm}, we also have a scale ambiguity: 
        if $ ( \w^\star, b^\star )$ is a solution,
        $ (  t \w^\star,  t b^\star ) $ with $ t > 0 $ is also 
        a solution.
        
        An important observation is that  the problem \eqref{eq:11b}
        must attain the optimum at 
        \begin{equation}
            \label{EQ:11c}
            b = \w^\T \bmu_2. 
        \end{equation}
        Otherwise if $ b >  \w^\T \bmu_2$, the optimal value
        of \eqref{eq:11b} must be smaller. 
        So we can rewrite \eqref{eq:11b} as an unconstrained problem
        \begin{equation}
            \label{eq:11d}
             \max_{\w  } \,\,
            \frac{ \w^\T ( \bmu_1 - \bmu_2 ) } 
                { \sqrt{ \w^\T \bSigma_1 \w } }.
        \end{equation}

        We have thus shown that, if $ \x_1 $ is distributed
        according to a symmetric, symmetric unimodal, or Gaussian distribution,
        the resulting optimization problem is identical. This is not surprising 
        considering the latter two cases 
        are merely special cases of the symmetric distribution family.

        At optimality, the inequality \eqref{eq:11} becomes an equation, and hence
        $ \gamma^\star$ can be obtained as in \eqref{eq:opt}.
        For ease of exposition, let us denote the fours cases in the right side of 
        \eqref{eq:6} as
        $ \varphi_{\rm gnrl} ( \cdot)$,
        $ \varphi_{\rm S} ( \cdot) $, 
        $ \varphi_{\rm SU} ( \cdot) $, and
        $ \varphi_{\cal G} ( \cdot) $. 
        For $ \gamma \in [0.5, 1) $, as shown in Fig.~\ref{fig:1},
        we have
            $  
                  \varphi_{\rm gnrl} ( \gamma )
               >  \varphi_{\rm S}  ( \gamma ) 
               >  \varphi_{\rm SU} ( \gamma )
               >  \varphi_{\cal G} ( \gamma ) 
               $. 
        Therefore, when solving \eqref{eq:opt} for $ \gamma^\star $, we have 
        $ 
          \gamma^\star_{\rm gnrl}  <
          \gamma^\star_{\rm S}  <      
          \gamma^\star_{\rm SU} <
          \gamma^\star_{\cal G}$.
        That is to say, one can get better accuracy when
        additional information about the data distribution is available, although
        the actual optimization problem to be solved is identical.
    \end{proof}
    
    \begin{figure}[t]
        \begin{center}
            \includegraphics[width=0.85\linewidth]{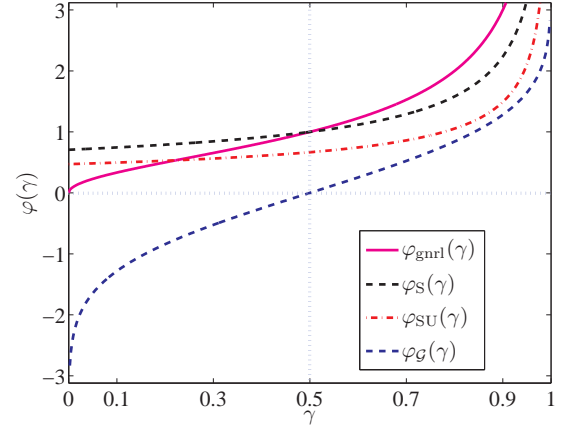}
        \end{center}
        \caption{The function $ \varphi(\cdot) $ in \eqref{eq:6}. The four curves correspond to
        the four cases. They are all monotonically increasing in $(0,1)$.}
        \label{fig:1}
    \end{figure}

        We have derived the biased MPM algorithm from a different perspective. 
        We reveal that only the assumption of symmetric distributions is needed
        to arrive at a simple unconstrained formulation.
        Compared the approach in \cite{huang2004mpm},
        we have used more information to simply the optimization problem. 
        More importantly, as well will show in the next section, this unconstrained
        formulation enables us to design a new boosting algorithm.

        There is a close connection between our algorithm 
        and the linear asymmetric classifier (LAC) in \cite{wu2005linear}.
        The resulting problem \eqref{eq:11d} is exactly the same as 
        LAC in \cite{wu2005linear}. 
        Removing the inequality in this constraint 
        leads to a problem solvable by eigen-decomposition.
        %
        %
        We have thus shown that the results of Wu \etal may be generalized from the Gaussian
        distributions assumed in \cite{wu2005linear} to symmetric distributions.

        %
        %
       %

        It is straightforward to kernelize 
        the linear classifier that we have discussed, following the work 
        of \cite{lanckriet2002mpm, huang2004mpm}.
        Here we are more interested, however, in designing a boosting algorithm that 
        takes the biased learning goal  into 
        consideration when selecting
        features.

\section{Linear Asymmetric Classification}
\label{sec:LAC}

        We have shown that starting from the biased minimax probability
        machine, we are able to obtain the same optimization formulation
        as shown in \cite{wu2005linear}, while much weakening the 
        underlying assumption (symmetric distributions versus Gaussian 
        distributions).
        Before we propose our LACBoost and FisherBoost, however, 
        we provide a brief overview of LAC.

        Wu {\em et al.} \cite{wu2008fast}   
        proposed linear asymmetric classification
        (LAC) as a post-processing step 
        for training nodes in the cascade framework. 
        In \cite{wu2008fast}, it is stated that
        LAC is guaranteed to reach an optimal solution
        under the assumption of Gaussian data distributions.
        We now know that this Gaussianality 
        condition may be relaxed.

        Suppose that we have a linear classifier 
        $ f(\x) = {\bf sign}(\w^\T \x - b)$.
        We seek a  $\{ \w , b \}$ pair with a very
        high accuracy on the positive data $\x_1$ and a moderate accuracy on
        the negative $\x_2$. 
        This can be expressed as the following problem:
    \begin{align}
        \max_{\w \neq {\bf 0}, b} 
        &
        \,  \Pr_{\x_1 \sim ( \bmu_1, \bSigma_1) }
        \{ \w ^\T \x_1 \geq b \}, \,\,
        \notag
        \\
        {\rm s.t.}  & \, \Pr_{\x_2 \sim (\bmu_2,\bSigma_2) }  
        \{ \w^\T \x_2
            \leq b \} = \lambda,
    \label{eq:LAC}
    \end{align}
    In \cite{wu2005linear},
    $\lambda$ is set to $0.5$ and it is assumed that for any $\w$, 
    $\w^\T \x_1$
    is Gaussian and $\w^\T \x_2$ is symmetric,  
    \eqref{eq:LAC} can be approximated by
    \eqref{eq:11d}.
    Again, these assumptions may be relaxed as we have shown in the last section.
    %
    %
    %
    \eqref{eq:11d} is similar to LDA's optimization problem
    \begin{equation}
        \label{EQ:LDA1}
        \max_{\w \neq \bf 0} \;\;
        \frac{  \w^\T ( \bmu_1 -  \bmu_2  ) } 
        {  \sqrt{  \w^\T ( \bSigma_1 + \bSigma_2 )  \w  }  }.
    \end{equation}
    \eqref{eq:11d} can be solved by eigen-decomposition and a close-formed 
    solution can be derived:
\begin{equation}
            \w^\star = \bSigma_1^{-1} ( \bmu_1 - \bmu_2 ),
    \quad 
    b^{\star} = { \w^{\star} } ^{\T} \bmu_2.
\label{eq:11d_SOL}
\end{equation}
On the other hand, each node in cascaded boosting classifiers has the following form:
\begin{equation}
    \label{EQ:nodeclassifier}
    f(\x) = {\bf sign}(\w^\T \H (\x) - b),
\end{equation}
 We override the symbol $ \H (\x)$ here, 
 which denotes the output vector of all weak classifiers over the datum $ \x $.
    We can cast each node as a linear classifier over the feature space
constructed by the binary outputs of all weak classifiers.
For each node in cascade classifier, we wish to maximize the detection
rate while maintaining the false positive rate at a
moderate level  (for example, around $50.0\%$). 
    That is to say,  the problem
    \eqref{eq:11d} represents the node learning goal. 
    Boosting algorithms such as AdaBoost can be used as feature
    selection methods, and LAC used to learn a linear classifier over
    those binary features chosen by boosting.
    The advantage of this approach is that LAC considers the asymmetric 
    node learning explicitly.

However, there is a precondition on the  validity of LAC that 
for any $\w$, $\w^\T \x_1$ is a Gaussian and $\w^\T \x_2$
is symmetric. 
In the case of boosting classifiers, $\w^\T \x_1$ and $\w^\T \x_2$ can be 
expressed as the margin of positive data and negative data, respectively.
Empirically  Wu {\em et al.} \cite{wu2008fast}
verified that $\w^\T \x$ is  approximately Gaussian for a cascade face detector.
We discuss this issue in more detail in Section~\ref{sec:exp}.
Shen and Li \cite{shen2010dual} theoretically proved that
under the assumption that weak classifiers are independent,  
the margin of AdaBoost follows the Gaussian distribution, 
as long as the number of weak classifiers is {\em sufficiently large}. 
In Section~\ref{sec:exp} we verify this theoretical result by performing the 
normality test on nodes with different number of weak classifiers.

\section{Constructing Boosting Algorithms from LDA and LAC}
\label{sec:LACBoost}

    In kernel methods, the original data are non\-linear\-ly 
    mapped to a feature space by a mapping
    function $ \Phi ( \cdot ) $. The function need not be known,
    however, as rather than being applied to the data directly,
    it acts instead through the inner product  
    $ \Phi ( \x_i ) ^\T \Phi ( \x_j )  $. 
    In boosting \cite{Ratsch2002BoostSVM}, however,
    the mapping function can be seen as being explicitly known,
    as
    $
        \Phi ( \x ) : \x \mapsto [ h_1(\x),\dots,h_n(\x) ]. 
    $
    Let us consider the Fisher LDA case first because the solution to LDA
    will generalize to LAC straightforwardly, by looking at
    the similarity between \eqref{eq:11d} and \eqref{EQ:LDA1}.

    Fisher LDA 
    maximizes the between-class variance and minimizes the within-class
    variance. In the binary-class case, the more general formulation in
     \eqref{EQ:LDA1} can be expressed as
    \begin{equation}
        \label{EQ:100}
        \max_\w \;\;  \frac{ ( \bmu_1 - \bmu_2 ) ^ 2 }
                         { \sigma_1 + \sigma_2 } 
            = 
                    \frac{   \w ^\T   \C_b \w }
                         {   \w ^\T  \C_w \w },
    \end{equation}
    where $ \C_b $ and $ \C_w $ are the between-class and within-class
    scatter matrices; $ \bmu_1 $ and $ \bmu_2 $ are
    the projected centers of the two classes.
    The above problem can be equivalently reformulated as 
    \begin{equation}
        \label{EQ:101}
        \min_\w \;\;  \w ^\T \C_w \w -  \theta ( \bmu_1 - \bmu_2  )
    \end{equation}
    for some certain constant $ \theta $ and under the assumption that
    $ \bmu_1 - \bmu_2 \geq 0 $.\footnote{In our face detection experiment,
    we found that this assumption could always be satisfied.}
    Now in the feature space, our data are 
    $ \Phi( \x_i ) $, $ i=1\dots m$.
    	Define the vectors $ \e, \e_1, \e_2 \in \Real^{m}$ such that $ \e = \e_1 + \e_2 $, 
		the $ i $-th entry of $ \e_1 $ is $1/m_1$ if $ y_i = +1 $ and $0$ otherwise, and 
		the $ i $-th entry of $ \e_2 $ is $1/m_2$ if $ y_i = -1 $ and $0$ otherwise.  
    We then see that
    \begin{align}
        \bmu_1
        &  = \frac{ 1 } { m_1 } \w^\T \sum_{y_i = 1}  \Phi(\x_i)  
           = \frac{ 1 } { m_1 } \sum_{y_i = 1} \A_{ i: } \w
           \notag
         \\
         & 
        = \frac{ 1 } { m_1 } \sum_{y_i = 1} (\A \w)_i
           = \e_1 ^\T \A \w   ,
    \end{align}
    %
    %
    %
    and
    \begin{align}
     \bmu_2
     & = 
     \frac{ 1 } { m_2 }   \w^\T \sum_{y_i = -1}  \Phi(\x_i)
     = \frac{ 1 } { m_2 } \sum_{y_i = -1} \H_{ i: } \w
     = - \e_2 ^\T  \A \w,
    \end{align}
    For ease of exposition we order the training data according to their
        labels so
        the vector $ \e \in \Real^{m}$:
        \begin{equation}
        \e = [ 1/m_1,\cdots, 1/m_2,\cdots  ]^\T,      
            \label{EQ:e}
        \end{equation}
        and the first $ m_1$ components of $ \brho $ correspond to the
        positive training data and the remaining ones
        correspond to the $ m_2$
        negative data.  
        We now see that $ \bmu_1 - \bmu_2 = \e^\T \brho  $, 
        $ \C_w =  {m_1 }/{ m } \cdot \bSigma_1 + {m_2 }/{ m } \cdot  \bSigma_2 $
        with
        $ \bSigma_{1,2} $ the covariance matrices. 
        Noting that
        \[
        \w^\T \bSigma_{1,2} \w = \frac{1}{m_{1,2} ( m_{1,2} - 1 ) }
        \sum_{i>k, y_i=y_k = \pm 1}
        (\rho_i - \rho_k )^2,
        \]
        we can easily rewrite the original problem \eqref{EQ:100} (and \eqref{EQ:101})
        into:
        \begin{align}
            \min_{\w,\brho}
             ~& 
            \tfrac{1}{2} \brho ^\T \Q \brho - \theta \e^\T
            \brho,
          \notag\\  
            \quad {\rm s.t.} ~&\w \psd {\bf 0},
             {\bf 1}^\T \w = 1,
      \notag \\
      &
     {\rho}_i = ( \A \w )_i,
        i = 1,\cdots, m.
            \label{EQ:QP1}
        \end{align}
        Here
        $ \Q = \begin{bmatrix} \Q_1 & {\bf 0} \\ {\bf 0} & \Q_2  \end{bmatrix} $
        is a block matrix with
        \[
        \Q_1 = 
        \begin{bmatrix}
                \tfrac{1}{m} & -\tfrac{1}{ m (m_1-1)} & \ldots & -\tfrac{1}{m(m_1-1)} \\
                -\tfrac{1}{m(m_1-1)} & \tfrac{1}{ m } & \ldots & -\tfrac{1}{m(m_1-1)} \\
                \vdots & \vdots & \ddots & \vdots \\
                -\tfrac{1}{m(m_1-1)} & -\tfrac{1}{m (m_1-1)} & \ldots &\tfrac{1}{m } 
        \end{bmatrix},
        \]
        and $ \Q_2 $ is similarly defined by replacing $ m_1$ with $ m_2 $ in $ \Q_1$:
        \[
        \Q_2 = 
            \begin{bmatrix}
                \tfrac{1}{m}         & -\tfrac{1}{m(m_2-1)} & \ldots &
                -\tfrac{1}{m(m_2-1)}                                \\
                -\tfrac{1}{m(m_2-1)} & \tfrac{1}{m}         & \ldots &
                -\tfrac{1}{m(m_2-1)}                                \\
                \vdots               & \vdots & \ddots & \vdots     \\
                -\tfrac{1}{m(m_2-1)} & -\tfrac{1}{m(m_2-1)} & \ldots
                &\tfrac{1}{m} 
            \end{bmatrix}. 
        \]
  Also note that we have introduced a constant $ \frac{1}{2} $ before the quadratic term
  for convenience. The normalization 
                  constraint $ { \bf 1 } ^\T \w = 1$
                  removes the scale ambiguity of $ \w $. Without it the problem is
                  ill-posed.

  We see from the form of \eqref{eq:11d} that the covariance of 
  the negative data is not involved in
  LAC and thus
  that if we set
   $ \Q = \begin{bmatrix} \Q_1 & {\bf 0} \\ {\bf 0} & \bf 0  \end{bmatrix} $ then
  \eqref{EQ:QP1} becomes the optimization problem of LAC.

%
%
        There may be extremely (or even infinitely) many weak
        classifiers in $ \cH $, the set from which
        $ h ( \cdot  ) $ is selected, meaning that the dimension of the optimization
        variable $ \w $ may also be extremely large.
    So \eqref{EQ:QP1} is a semi-infinite quadratic program (SIQP).
    We show how column generation can be used to solve this problem.
    To make column generation applicable, we need to derive a
    specific Lagrange dual of the primal
    problem.

\subsection{The Lagrange Dual Problem}

    We now derive the Lagrange dual of the quadratic problem \eqref{EQ:QP1}.    
    Although we are only interested in the variable $ \w $, we need to
    keep the auxiliary variable $ \boldsymbol  \rho $ in order to obtain
    a meaningful dual problem. The Lagrangian of \eqref{EQ:QP1}
    is
    \begin{align*}
        L (  
           \underbrace{ \w, \brho}_{\rm primal}, \underbrace{ \u, r }_{\rm dual}
        ) 
    & = \tfrac{1}{2} \brho ^\T \Q \brho   -  \theta \e^\T \brho 
    + \u ^\T ( \brho - \A \w  ) - \q ^\T \w 
    \notag
    \\ 
    & + r ( {\bf 1} ^\T \w - 1 ), 
    \end{align*} 
    with 
     $
     \q \psd \bf 0 
     $.  
    $ \sup_{\u, r} \inf_{ \w, \brho } L ( \w, {\brho}, \u, r  ) $
    gives the following  Lagrange dual:
           \begin{align}
               \max_{\u, r} ~& -r - \overbrace{
                       \tfrac{1}{2} 
                      (\u - \theta \e)^\T \Q^{-1} (\u - \theta \e)
                      }^{\rm regularization},  
                      %
                      %
            {\rm \;\; s.t.} 
                     ~
                     %
                     %
                      \sum_{i=1}^m u_i \A_{i:} \nsd r {\bf 1 } ^\T. 
            \label{EQ:dual}
        \end{align}
        In our case, $ \Q $ is rank-deficient and its inverse does not exist
        (for both LDA and LAC).
        We can simply regularize $ \Q $ with $ \Q + \delta {\bf I} $ with  
        $ \delta $ a  small positive constant.
        Actually, $ \Q $ is a diagonally dominant matrix but not strict diagonal dominance.
        So $ \Q + \delta {\bf I} $ with any $ \delta > 0 $ is strict 
        diagonal dominance and by the Gershgorin circle theorem, 
        a strictly diagonally   dominant matrix must be invertible. 
        
        One of the  KKT optimality conditions between the dual and primal
        \begin{equation}
         \brho^\star = - \Q^{-1} ( \u ^ \star - \theta \e ),
        \end{equation}
        which can be used to establish the connection between the dual optimum and
        the primal optimum. 
        This is obtained by the fact that 
        the gradient of $ L $ w.r.t. $ \brho $ must vanish at 
        the optimum, $ { \partial L } / { \partial \rho_i } = 0  $,
        $ \forall i = 1\cdots n $.

        Problem \eqref{EQ:dual} can be viewed as a regularized LPBoost problem.
        Compared with the hard-margin LPBoost \cite{Demiriz2002LPBoost},
        the only difference is the regularization term in the cost function.
        The duality gap between the primal \eqref{EQ:QP1} and the 
        dual \eqref{EQ:dual} is zero. In other words, the solutions of
        \eqref{EQ:QP1} and \eqref{EQ:dual} coincide. 
        Instead of solving \eqref{EQ:QP1} directly, one calculates the
        most violated constraint in \eqref{EQ:dual} iteratively for
        the current solution and adds this constraint to the
        optimization problem.  In theory, any column that violates
        dual feasibility can be added.  To speed up the convergence,
        we add the most violated constraint by solving the following
        problem:
      \begin{equation}
          h' ( \cdot ) =  {\rm argmax}_{h( \cdot ) } ~ 
            \sum_{i=1}^m u_i y_i h ( \x_i).
         \label{EQ:pickweak}
      \end{equation}
        This  is exactly the same as the one that standard AdaBoost
        and LPBoost use for producing the best weak classifier. That
        is to say, to find the weak classifier that has minimum weighted
        training error.  We summarize the LACBoost/FisherBoost
        algorithm in
        Algorithm~\ref{alg:QPCG}.
        By simply changing  $ \Q_2 $, Algorithm~\ref{alg:QPCG} can be used to
        train either LACBoost or FisherBoost.
        Note that to obtain an actual strong classifier,
        one may need to include an offset $ b $, {\em i.e.} the final classifier
        is $ \sum_{j=1}^n h_j (\x) - b $ because from the cost function
        of our algorithm \eqref{EQ:101}, we can see that the cost function itself
        does not minimize any classification error. It only finds a projection
        direction in which the data can be maximally separated. A simple line
        search can find an optimal $ b $.  
        Moreover, when training a cascade, we need to tune this offset anyway
        as shown in \eqref{EQ:nodeclassifier}.

        The convergence of Algorithm~\ref{alg:QPCG} is guaranteed by
        general column generation or cutting-plane algorithms, which
        is easy
        to establish.  When a new $ h'(\cdot) $ that violates dual
        feasibility is added, the new optimal value of the dual
        problem (maximization) would decrease.  Accordingly, the
        optimal value of its primal problem decreases too because they
        have the same optimal value due to zero duality gap. Moreover
        the primal cost function is convex, therefore in the end it
        converges to the global minimum.

   \linesnumbered\SetVline
   \begin{algorithm}[t]
   \caption{Column generation for SIQP.}   
   \centering
   {\small
   \begin{minipage}[]{.94\linewidth}
   \KwIn{Labeled training data $(\x_i, y_i), i = 1\cdots m$;
         termination threshold $ \varepsilon > 0$;
         regularization
         parameter $ \theta $; maximum number of iterations
         $ n_{\rm max}$.
    }
       { {\bf Initialization}:
            $ m = 0 $;
            $ \w = {\bf 0} $;
            and $ u_i = \frac{1}{ m }$, $ i = 1$$\cdots$$m$. 
   }

   \For{ $ \mathrm{iteration} = 1 : n_\mathrm{max}$}
   {
     %
     %
     \ADot
         Check for the optimality: \\
         {\bf if}{ $ \mathrm{iteration}  > 1 $ \text{ and } $
                    \sum_{ i=1 }^m  u_i y_i h' ( \x_i )  
                           < r + \varepsilon $},
                  \\
                  { \bf then}
                  \\
                  $~ ~ ~$ break;  and the problem is solved; 
        
     \ADot
         Add $  h'(\cdot) $ to the restricted master problem, which
         corresponds to a new constraint in the dual;
     %
     %
     %

      \ADot  
         Solve the dual problem \eqref{EQ:dual}
         (or the primal problem \eqref{EQ:QP1}) 
         and update $ r $ and
         $ u_i$ ($ i = 1\cdots m$). 
%
%

      \ADot  
         Increment the number of weak classifiers
             $n = n + 1$.  
%
%
%
   }
   \KwOut{
         %
         The selected features are $ h_1, h_2, \dots, h_n $.
         The final strong classifier is:
         $ F ( \x ) = \textstyle \sum_{j=1}^{ n } w_j h_j( \x ) - b $.
         Here the offset $ b $ can be learned by a simple search. 
     
   }
   \end{minipage}
   } 
   \label{alg:QPCG}
   \end{algorithm}

    At each iteration of column generation,
    in theory, we can  solve either the dual \eqref{EQ:dual} 
    or the primal problem \eqref{EQ:QP1}. 
    However, 
    in practice, it could be  much faster to solve the primal problem because
     \begin{enumerate}
         \item
  Generally,
    the primal problem has a smaller size, hence faster to solve.
    The number of variables of \eqref{EQ:dual} is $ m $ at each iteration,
    while the number of variables is the number of iterations   
    for the primal problem. 
    For example, in Viola-Jones' face detection framework, 
    the number of training data $ m = 
    10,000 $ and  $ n_{\rm max} = 200 $. In other words, the  
    primal problem has at most $ 200 $ variables in this case;

    \item
   
        The dual problem is a standard QP problem. It has no special structure
        to exploit. As we will show, the primal problem belongs to
        a special class of problems and
        can be efficiently 
        solved using entropic/exponentiated 
        gradient descent (EG) \cite{Beck03Mirror,Globerson07Exp}. 
        A fast QP solver is extremely important for training a 
        object detector because we need to the solve a few thousand 
        QP problems. 
     \end{enumerate}

    We can recover both of the dual variables 
    $ \u^\star, r^\star $ easily from 
    the primal variable $ \w^\star $:
    \begin{align}
        \u^\star &=  - \Q\brho^\star  + \theta \e; \label{EQ:KA}\\
         r^\star  &=   \max_{ j = 1 \dots n }  
         \bigl\{ \textstyle \sum_{i=1}^m u_i^\star \A_{ij} \bigr\}.
         \label{EQ:KAb}
    \end{align}
    The second equation is obtained by the fact that 
     in the dual problem's  constraints, at optimum,
    there must exist at least one  $ u_i^\star$ 
    such that the equality holds. That is to say,
    $ r^\star $ is the largest {\em edge}
    over all weak classifiers.

    We give a brief introduction to the EG algorithm before we proceed. 
    Let us first define the unit simplex 
    $ \Delta_n =  \{ 
    \w \in \Real^n  :  {\bf 1 } ^ \T \w = 1, \w \psd {\bf 0 }
    \} $. 
    EG efficiently solves the convex optimization problem
    \begin{equation}
        \label{EQ:EG1}
        \min_\w \,\,\, f(\w), \,
        {\rm s.t.} \,\, \w \in \Delta_n,  
    \end{equation}
    under the assumption that the objective function $ f(\cdot) $
    is a convex Lipschitz continuous function with Lipschitz
    constant $ L_f $ w.r.t. a fixed given norm $ \lVert \cdot \rVert$.
    The mathematical definition of $ L_f $ is that
    $  | f(\w) -f (\z) |  \leq L_f   \lVert  \x - \z  \rVert$ holds
    for any $ \x, \z $ in the domain of $ f(\cdot)$.
    The EG algorithm is very simple:
    \begin{enumerate}
    \item
        Initialize with $\w^0 \in \text{the interior of }  \Delta_n$;
    \item
        Generate the sequence $ \{ \w^k \} $, $ k=1,2,\cdots$
        with:
        \begin{equation}
            \label{EQ:EQ2}
            \w^k_j = \frac{ \w^{k-1}_j \exp [ - \tau_k f'_j ( \w^{k-1} ) ]  } 
            { \sum_{j=1}^n  \w^{k-1}_j \exp [ - \tau_k f'_j ( \w^{k-1} ) ] }. 
        \end{equation}
        Here $ \tau_k $ is the step-size. 
        $ f'( \w ) = [ f_1'(\w), \dots,  f_n'(\w) ] ^\T $
        is the gradient of $ f(\cdot) $;
    \item
        Stop if some stopping criteria are met.
    \end{enumerate}
    The learning step-size can be determined by 
        $    \tau_k = \frac{ \sqrt{ 2\log n } } { L_f }
                     \frac{1}{ \sqrt{ k } },
        $
    following \cite{Beck03Mirror}.
    In \cite{Globerson07Exp}, the authors have 
    used a simpler strategy to set the learning rate. 

    EG is a very useful tool for solving large-scale 
    convex minimization problems over the unit simplex. 
    Compared with standard QP solvers like Mosek 
    \cite{Mosek}, EG is much faster. EG makes it possible 
    to train a detector using almost the same amount of time
    as using standard AdaBoost as the majority of time is
    spent on weak classifier training and bootstrapping.

    In the case that $ m_1 \gg 1 $, 
    \[
            \Q_1 =
            \frac{1}{m}
    \begin{bmatrix}
        1   & -\tfrac{1}{  m_1-1} & \ldots & -\tfrac{1}{  m_1-1 } \\
        -\tfrac{1}{ m_1-1 } & 1 & \ldots & -\tfrac{1}{ m_1-1} \\
        \vdots & \vdots & \ddots & \vdots \\
        -\tfrac{1}{ m_1-1 } & -\tfrac{1}{ m_1-1 } & \ldots &  1 
    \end{bmatrix} \approx  \frac{1}{m} \bf I.
    \]
    Similarly, for LDA, $ \Q_2 \approx  \frac{1}{m} \bf I$
    when $ m_2 \gg 1 $. Hence,
    \begin{equation}
     \label{EQ:Q10}    
      \Q \approx 
     \begin{cases}
          \frac{1}{m} \bf I;    & \text{for Fisher LDA},    \\
           \frac{1}{m} \begin{bmatrix}
                                          {\bf I} & {\bf 0} \\
                                          {\bf 0} & {\bf 0}
                                 \end{bmatrix},
                                & \text{for LAC}.
    \end{cases}
    \end{equation}
    Therefore, the problems involved can be simplified when $ m_1 \gg 1 $ and
    $ m_2 \gg 1 $ hold.
    The primal problem \eqref{EQ:QP1} equals
     \begin{align}
            \min_{\w,\brho} ~& \tfrac{1}{2} \w ^\T ( \A^\T \Q \A) \w  
            - ( \theta \e ^\T
            \A ) \w,
            %
            %
            \;\;
            {\rm s.t.}
            \;
            %
            %
            \w \in \Delta_n.
            \label{EQ:QP2}
        \end{align}
        We can efficiently solve \eqref{EQ:QP2} 
        using the EG method. 
        In EG there is an important parameter $ L_f $, which is
        used to determine the step-size. 
        $ L_f $ can be determined by the $\ell_\infty $-norm of $ | f' (\w) | $.
        In our case $ f' (\w) $ is a linear function, which is trivial to compute.
        The convergence of EG is guaranteed; see \cite{Beck03Mirror} for details.
        
        In summary, when using EG to solve the primal problem, 
        Line $ 5 $ of Algorithm~\ref{alg:QPCG} is:        
        
        \ADot 
        {\em Solve the primal problem \eqref{EQ:QP2} using EG, and update 
        the dual variables $ \u $ with \eqref{EQ:KA}, and $ r $ with \eqref{EQ:KAb}.
        }

\section{Experiments}
\label{sec:exp}

    In this section, we first show an experiment on toy data
    and then apply the proposed methods to face detection 
    and pedestrian detection.

\subsection{Synthetic Testing}

\begin{figure}[t!]
    \centering
        \includegraphics[width=.35\textwidth]{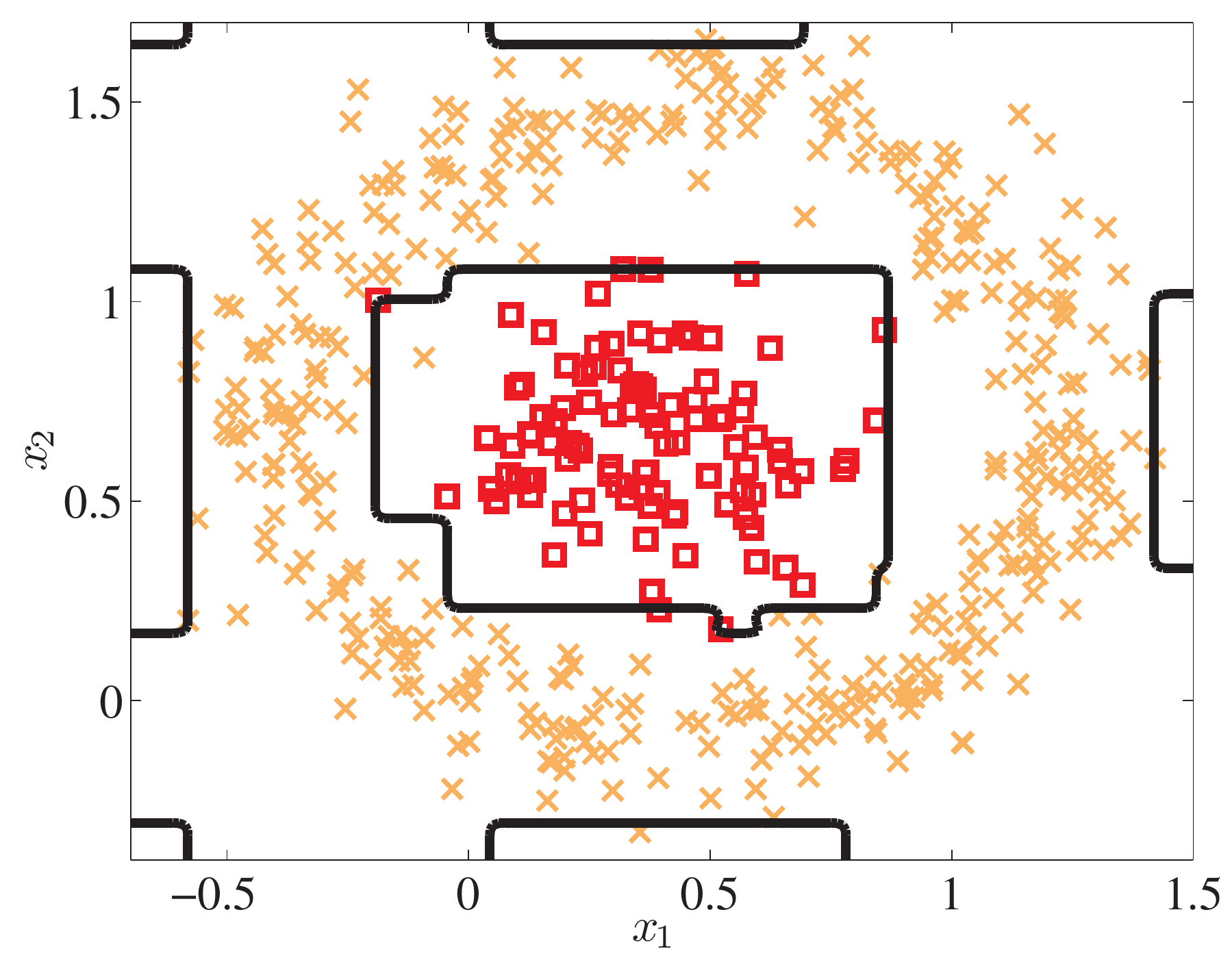}
        \includegraphics[width=.35\textwidth]{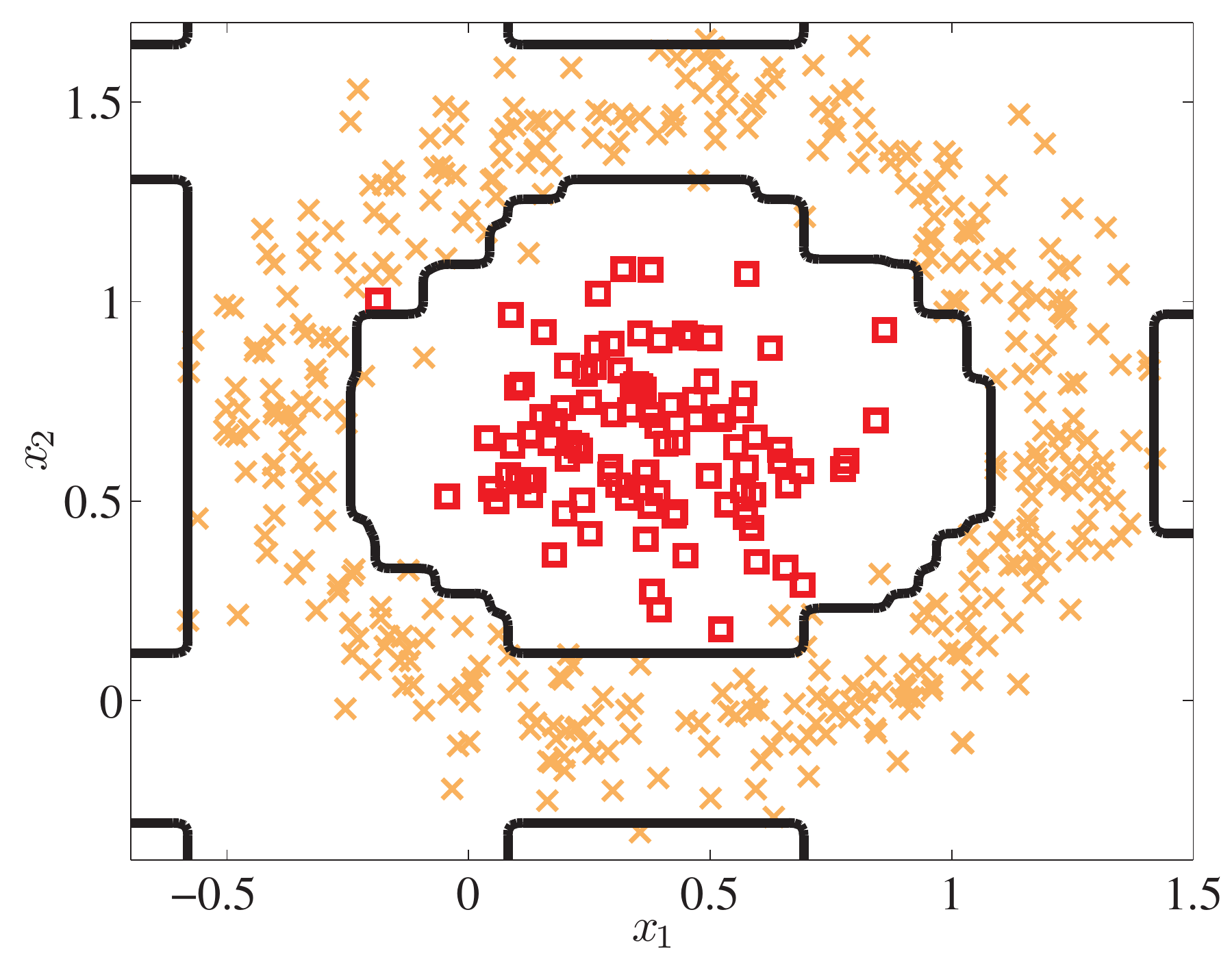}
    \caption{Decision boundaries of 
    AdaBoost (top) and FisherBoost (bottom) on $2$D artificial data
    (positive data represented by $ \square $'s and negative data
    by $\times$'s). Weak classifiers are decision stumps.
    In this case, FisherBoost intends to correctly classify more
    positive data in this case. 
    }
    \label{fig:toy}
\end{figure}

    First, let us show a simple example on a synthetic dataset 
    (more negative data than positive data)
    to illustrate
    the difference between FisherBoost and AdaBoost.
    Fig. \ref{fig:toy} demonstrates the subtle difference of the classification 
    boundaries obtained by AdaBoost and FisherBoost when applied to these data. 
    We can see that
    FisherBoost 
    places more emphasis 
    on correctly classifying positive data points than does AdaBoost.
    This might be due to the fact that AdaBoost only optimizes the overall 
    classification accuracy. 
    This finding is consistent with the result in \cite{Paisitkriangkrai2009CVPR}.
    
\begin{figure*}[t]
    \centering
        \includegraphics[width=.32\textwidth]{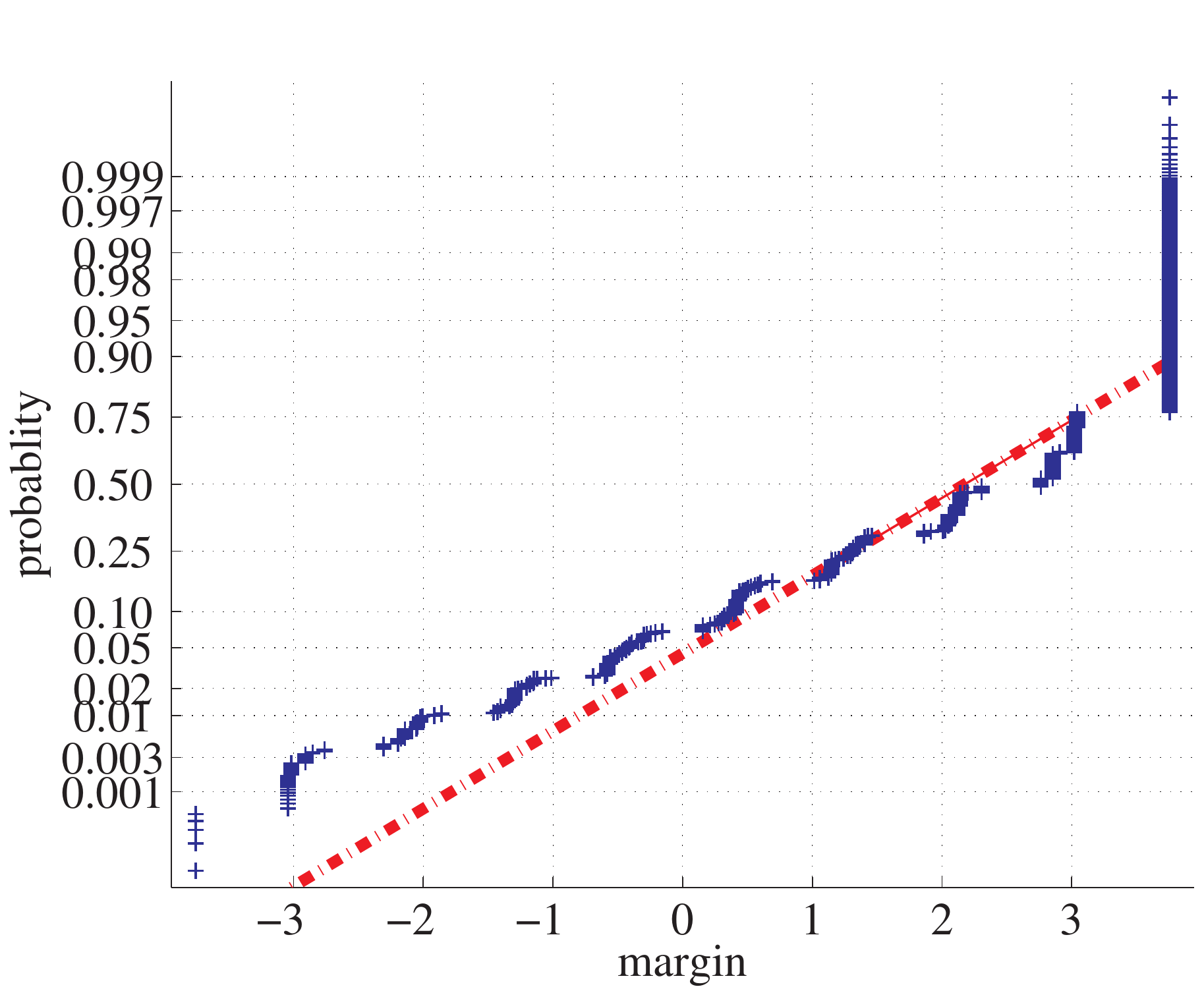}
        \includegraphics[width=.32\textwidth]{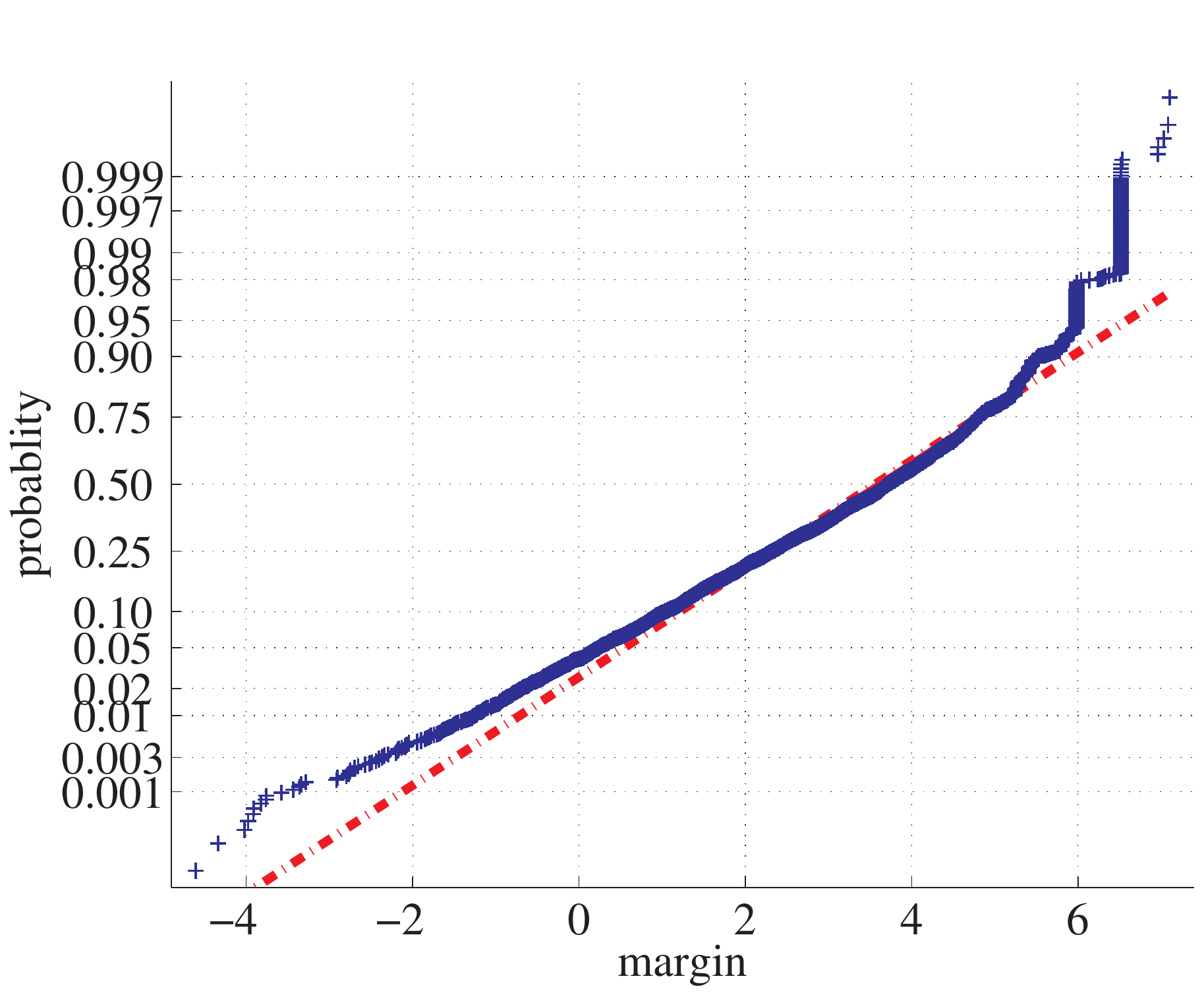}
        \includegraphics[width=.32\textwidth]{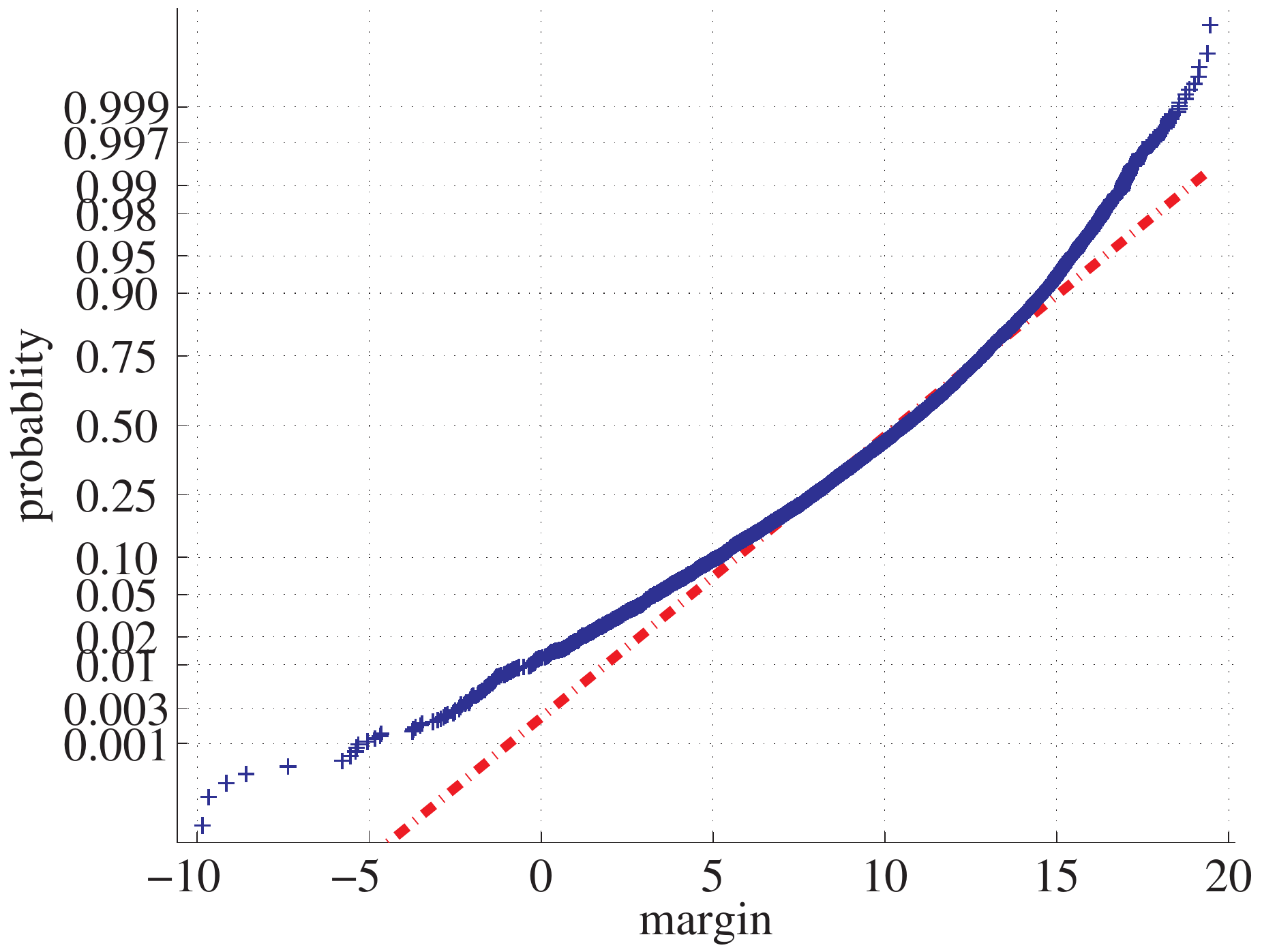}
    \caption{Normality test (normal probability plot)
    for the face data's margin distribution of nodes $1$, $2$, $3$.
    The $ 3 $ nodes contains $ 7 $, $ 22 $, $ 52 $ weak classifiers respectively.
    Curves close to a straight line mean close to a Gaussian.
    }
    \label{fig:normplot}
\end{figure*}

\subsection{Face Detection Using a Cascade Classifier}

    In this section, we compare FisherBoost and LACBoost with the 
    state-of-the-art in face detection.

    We first show some results about the validity 
    of LAC (and Fisher LDA) post-processing for improving 
    node learning in object detection.

\input{alg_multiexit_cascade.tex}
   The algorithm for training a multi-exit cascade is shown 
   in Algorithm \ref{alg:MultiExit_LAC}.


        As is described above, LAC and LDA assume that the margins of training data
        associated with the node classifiers in such a
        cascade exhibit a Gaussian distribution.
	    In order to evaluate the degree to which this is true for the face detection task we show in
        Fig. \ref{fig:normplot} normal probability plots of the margins of the positive training
        data for each of the first three node classifiers in a multi-exit LAC cascade.  The figure shows that
        the larger the number of weak classifiers used the more closely the margins follow a
        Gaussian distribution.  From this we infer that LAC, Fisher LDA post-processing, 
        (and thus LACBoost and FisherBoost) can be
        expected to achieve a better performance when a larger number of weak classifiers are used.
        We therefore apply LAC/LDA only within the later nodes (for example, $9$ onwards) 
        of a multi-exit cascade
        as these nodes contain more weak classifiers. 
    Because the late nodes of a multi-exit cascade contain more weak classifiers than the standard
    Viola-Jones' cascade, 
    we conjecture that the multi-exit cascade 
    might meet the Gaussianity requirement better. 
    We have compared 
    multi-exit cascades with LDA/LAC post-processing against 
    standard cascades with LDA/LAC post-processing in \cite{wu2008fast}
    and slightly improved performances were obtained.

    Six methods are evaluated with  the multi-exit cascade framework
    \cite{pham08multi},
    which are AdaBoost with LAC post-processing,
    or LDA post-processing,
    AsymBoost with LAC or LDA post-processing \cite{wu2008fast}, and our
    FisherBoost,
    LACBoost. We have also implemented Viola-Jones'
    face detector as the baseline \cite{viola2004robust}.
    As in \cite{viola2004robust}, five  basic types of Haar-like features 
    are calculated, resulting in a $162, 336$ dimensional over-complete
    feature set on an image of $24 \times 24$ pixels.
    To speed up the weak classifier training, as in \cite{wu2008fast},
    we  uniformly sample $10\%$ of features for
    training weak classifiers (decision stumps). 
    The training data are $9,832$ mirrored $24 \times 24$ face images 
    ($5,000$ for training and $4,832$ for validation) and $7,323$ large background images,
    which are the same as in \cite{wu2008fast}. The face images for training
    are provided by Viola and Jones' work---the same
    as the face training data used in \cite{viola2004robust}.

\begin{figure*}[t]
    \begin{center}
        \includegraphics[width=.2438\textwidth]{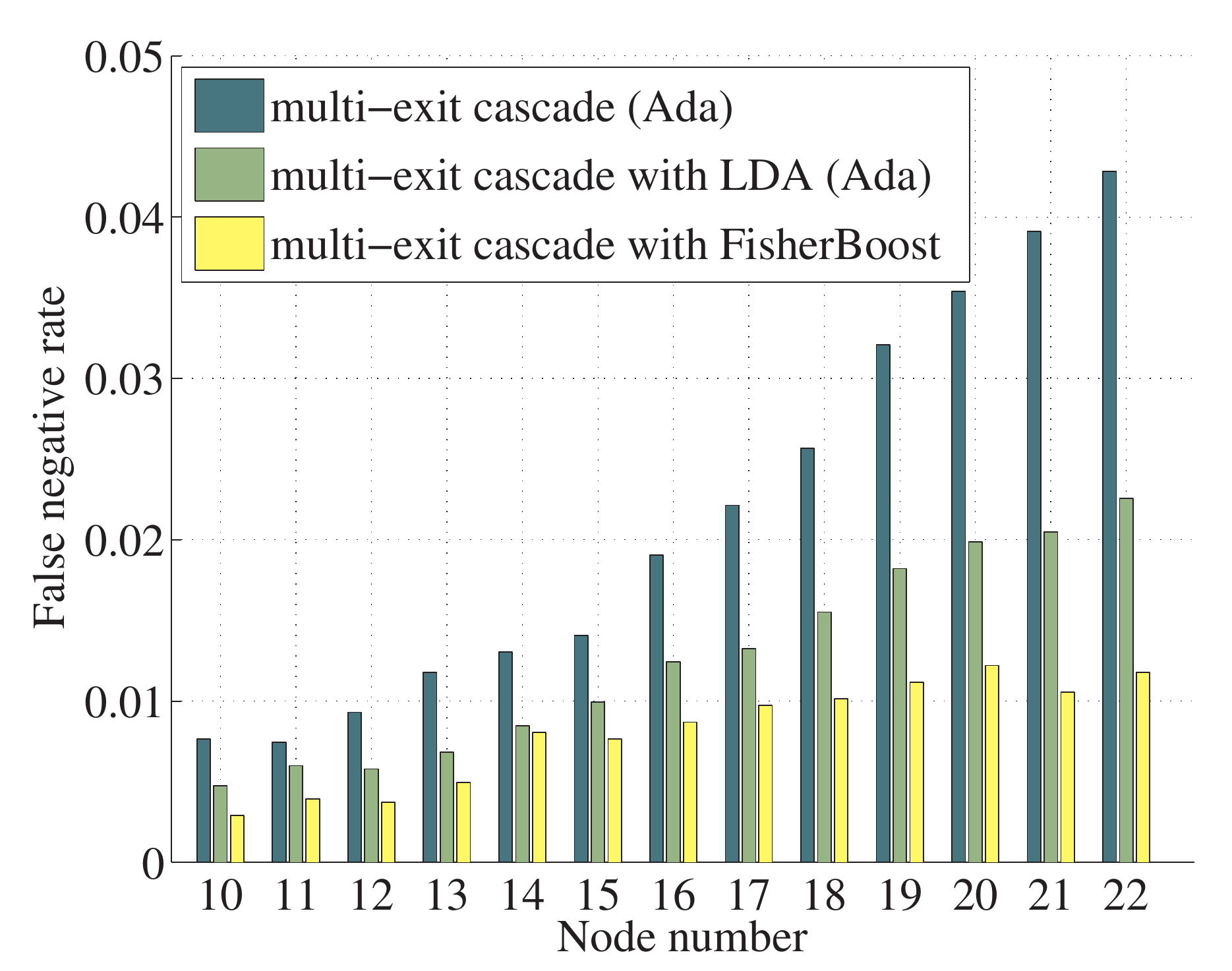}
        \includegraphics[width=.2438\textwidth]{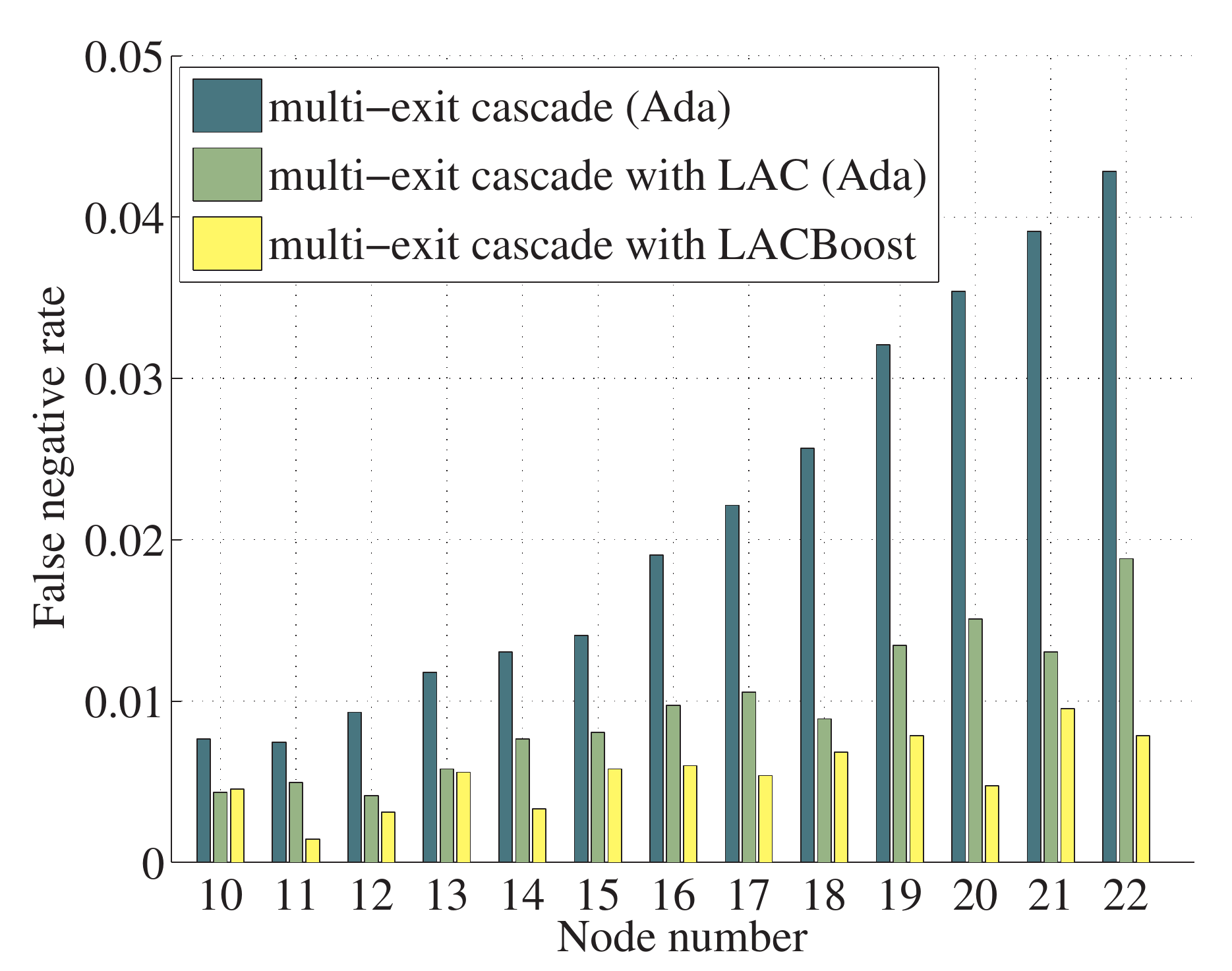}
        \includegraphics[width=.2438\textwidth]{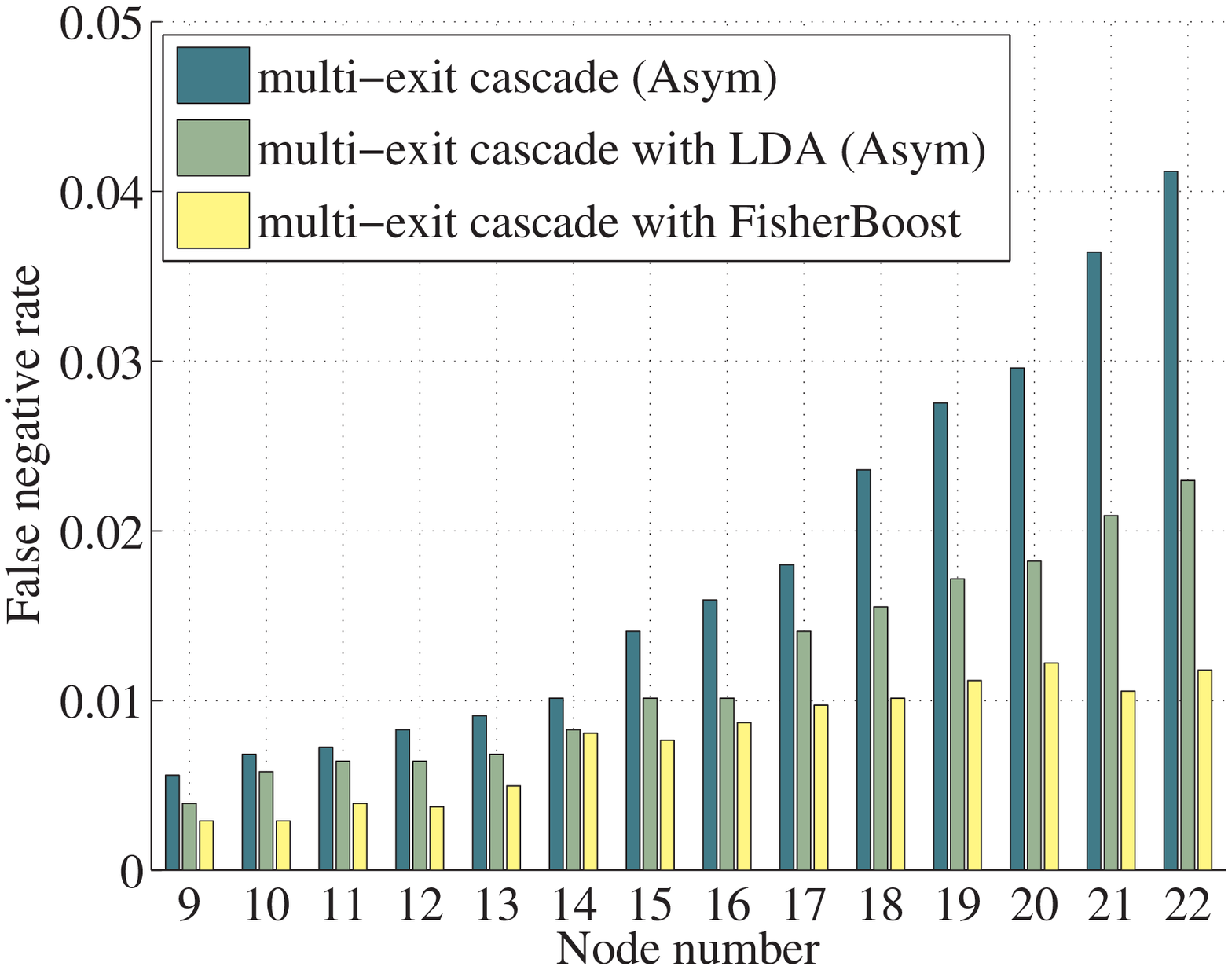}
        \includegraphics[width=.2438\textwidth]{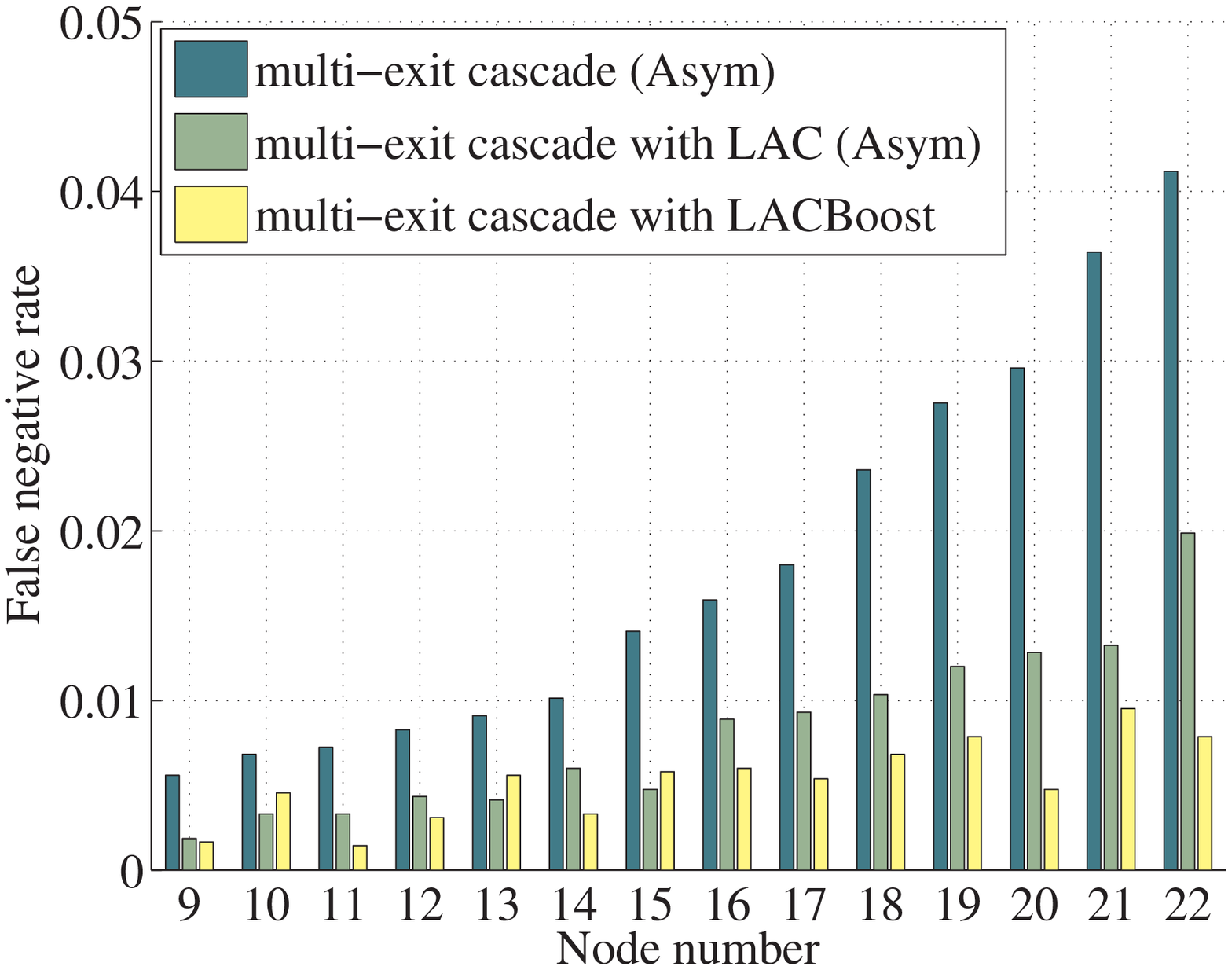}
    \end{center}
    \caption{Node performances on the validation data for face detection. 
       ``Ada'' means that features are selected using AdaBoost;
    ``Asym'' means that features are selected using AsymBoost.
    }
    \label{fig:node1}
\end{figure*}

    Multi-exit cascades with $22$ exits and $2,923$ weak classifiers are
    trained with each of the methods listed above.
       In order to ensure a fair comparison,
       we have used the same cascade structure and 
    same number of weak classifiers for all the compared learning methods.
    The indexes of exits are pre-set to simplify the training
    procedure.

      For our FisherBoost and LACBoost, we have an important parameter  
      $\theta$, which is chosen from 
      $\{
      \frac{1}{10}, 
      \frac{1}{12},
      \frac{1}{15},
      $
      $
      \frac{1}{20}, 
      \frac{1}{25},
      \frac{1}{30},
      $
      $
      \frac{1}{40},
      \frac{1}{50} 
      \}$.
      We have not carefully tuned this parameter using cross-validation.
      Instead, we train a $10$-node cascade for each candidate $ \theta$, 
      and choose the one with the best {\em training} 
      accuracy.\footnote{To train a complete $22$-node cascade
      and choose the best $ \theta $
      on cross-validation data may give better detection rates.} 
     At each exit, negative examples misclassified by current cascade are
     discarded, and new negative examples are bootstrapped from the background
     images pool.  
     In total, billions of negative examples are extracted from the pool.
        The positive training data and validation data keep unchanged  during the
        training process.

Our experiments are performed on a workstation with $8$ Intel Xeon
E$5520$ CPUs and $32$GB RAM.
It takes about $3$ hours  to train the multi-exit cascade with AdaBoost or AsymBoost.
For FisherBoost and LACBoost, it takes less than $ 4 $ hours to train
a complete multi-exit cascade.\footnote{Our implementation is in C++ and 
                                        only the weak classifier
                                        training part is parallelized using OpenMP. 
}
    In other words, 
    our EG algorithm takes less than $ 1 $ hour to
    solve the primal QP problem (we need to solve a QP at each iteration).
    As an estimation of the computational complexity, 
    suppose that the number of training
    examples is $  m $, number of weak classifiers is $ n $.
%
%
%
    At each iteration of the cascade training,
    the complexity of solving the primal QP using EG is
    $  O( m n  + k n^2) $ with $ k $ the iterations 
    needed for EG's convergence.
    The complexity  for  training the weak classifier  is
    $ O( m  d ) $ with $d$ the number of all Haar-feature patterns.
    In our experiment, $ m = 10,000 $,
    $ n \approx 2900 $,
    $d = 160,000$,
    $ k < 500 $.
      So the majority of the computational cost of the training process 
      is bound up in the weak classifier training.

    We have also experimentally observed the speedup of EG against standard QP solvers. 
    We solve the primal QP defined by \eqref{EQ:QP2} using EG and Mosek \cite{Mosek}.
    The QP's size is $ 1,000 $ variables. 
    With the same accuracy tolerance (Mosek's primal-dual gap is set to $ 10^{-7}$
    and EG's convergence tolerance is also set to $ 10^{-7}$), 
    Mosek takes $1.22 $ seconds and EG is
    $ 0.0541 $ seconds on our standard Desktop.
    So EG is about $ 20 $ times faster. 
    Moreover,  at iteration $ n + 1 $ of training the cascade,
    EG can take advantage of the last iteration's solution
    by starting EG from a small perturbation of the previous solution. 
    Such a warm-start gains a $ 5 $ to $ 10\times $ speedup in our experiment,
    while there is no off-the-shelf warm-start QP solvers available yet.

    We evaluate the detection performance on the MIT+CMU frontal
    face test set. 
This dataset is made up of $507$ frontal faces in $130$ images with 
different background.

If one positive output has less than $50\%$ variation of shift and
scale from the ground-truth, we treat it as a true positive, otherwise
a false positive.

In the test phase, 
the scale factor of the scanning window is set to $1.2$ and 
the stride step is set to $1$ pixel.
Two performance metrics are used here: one for each node and one for 
the entire cascade.
The node metric is how well the classifiers meet the node learning objective,
which provides useful information about 
the capability of each method to achieve the node learning goal.
The cascade metric uses the receiver operating characteristic (ROC)
to compare the entire cascade's performance.
Note that multiple factors impact on the cascade's performance, however,
including: the classifier set,
the cascade structure,  bootstrapping {\em etc}.

    Fig.~\ref{fig:node1} shows the false-negative rates for the various forms of node classifiers
    when applied to the MIT+CMU face data.  The figure shows that FisherBoost and LACBoost exhibit
    significantly better node classification performance than the post-processing approach, which
    verifies the advantage of selecting features on the basis of the node learning goal.  Note that
    the performance of FisherBoost and LACBoost is very similar, but also that LDA or LAC
    post-processing can considerably reduce the false negative rates over the standard Viol-Jones'
    approach, which corresponds with the findings in~\cite{wu2008fast}.

%
%

    The ROC curves in Fig.~\ref{fig:ROC1} demonstrate the superior performance of FisherBoost and
    LACBoost in the face detection task.  Fig.~\ref{fig:ROC1} also shows that LACBoost does not
    outperform FisherBoost in all cases, in contract to the node performance 
    (detection rate) results.  Many factors
    impact upon final performance, however, and these sporadic results are not seen as being
    particularly indicative.
    One possible cause is that LAC makes the assumption of Gaussianity 
    and symmetry data distributions,
    which may not hold well in the early nodes. 
    Wu {\em et al.} have observed the same phenomenon that LAC post-processing
    does not outperform LDA post-processing in some cases.

    The error reduction results of FisherBoost and LACBoost in 
    Fig.~\ref{fig:ROC1} are not as great as those in Fig. \ref{fig:node1}.
    This might be explained by the fact that the
    cascade and negative data bootstrapping 
    are compensating for the inferior node classifier performance to some extent.
   
    We have also compared our methods with the boosted greedy sparse LDA (BGSLDA) in
    \cite{Paisitkriangkrai2009CVPR,GSLDA2010Shen},
    which is considered one of the state-of-the-art.
    %
    %
    FisherBoost and LACBoost outperform 
    BGSLDA with AdaBoost/AsymBoost 
    in the detection rate. 
    Note that BGSLDA uses the standard cascade.  

\begin{figure*}[t]
    \begin{center}
        \includegraphics[width=.38\textwidth]{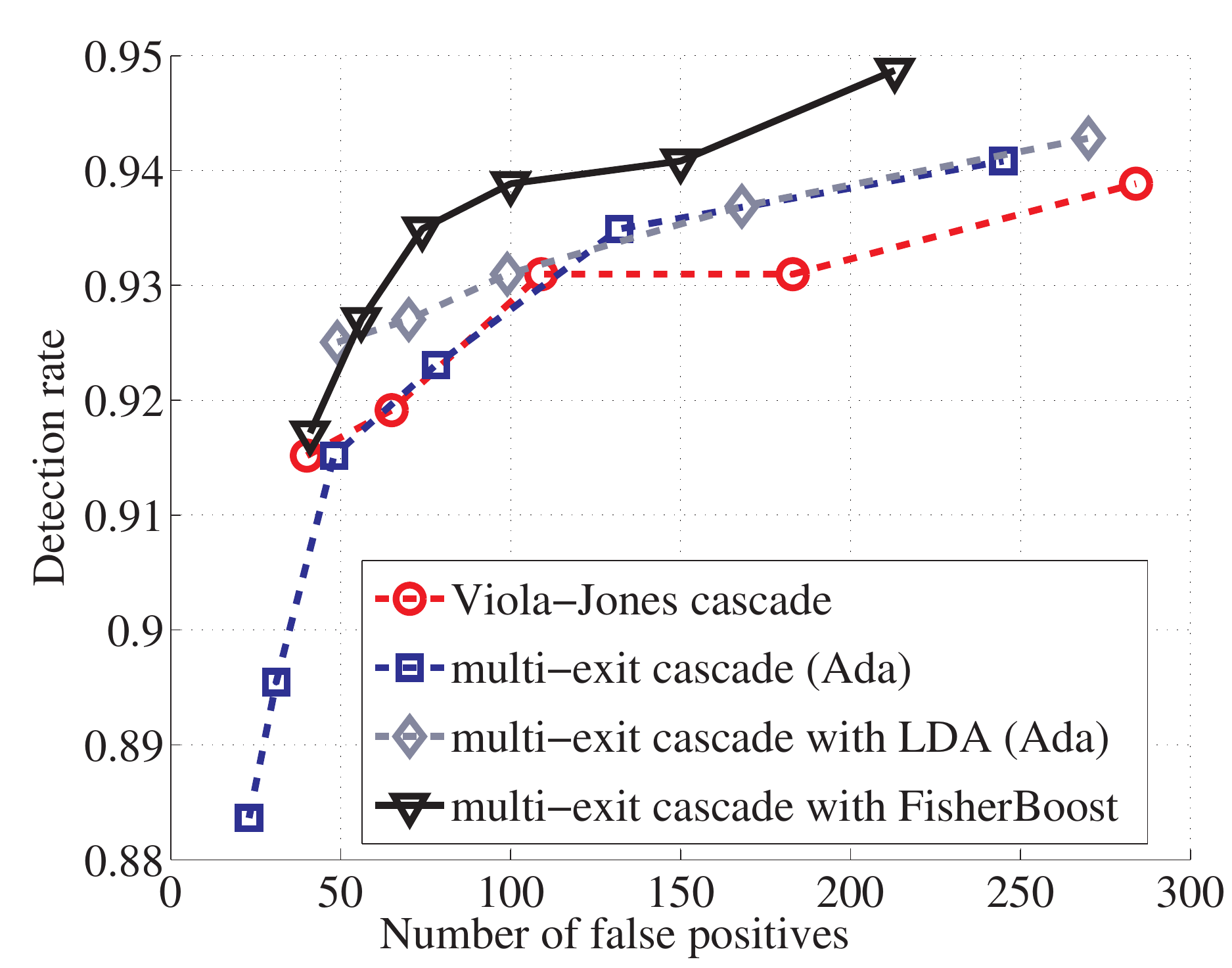}
        \includegraphics[width=.38\textwidth]{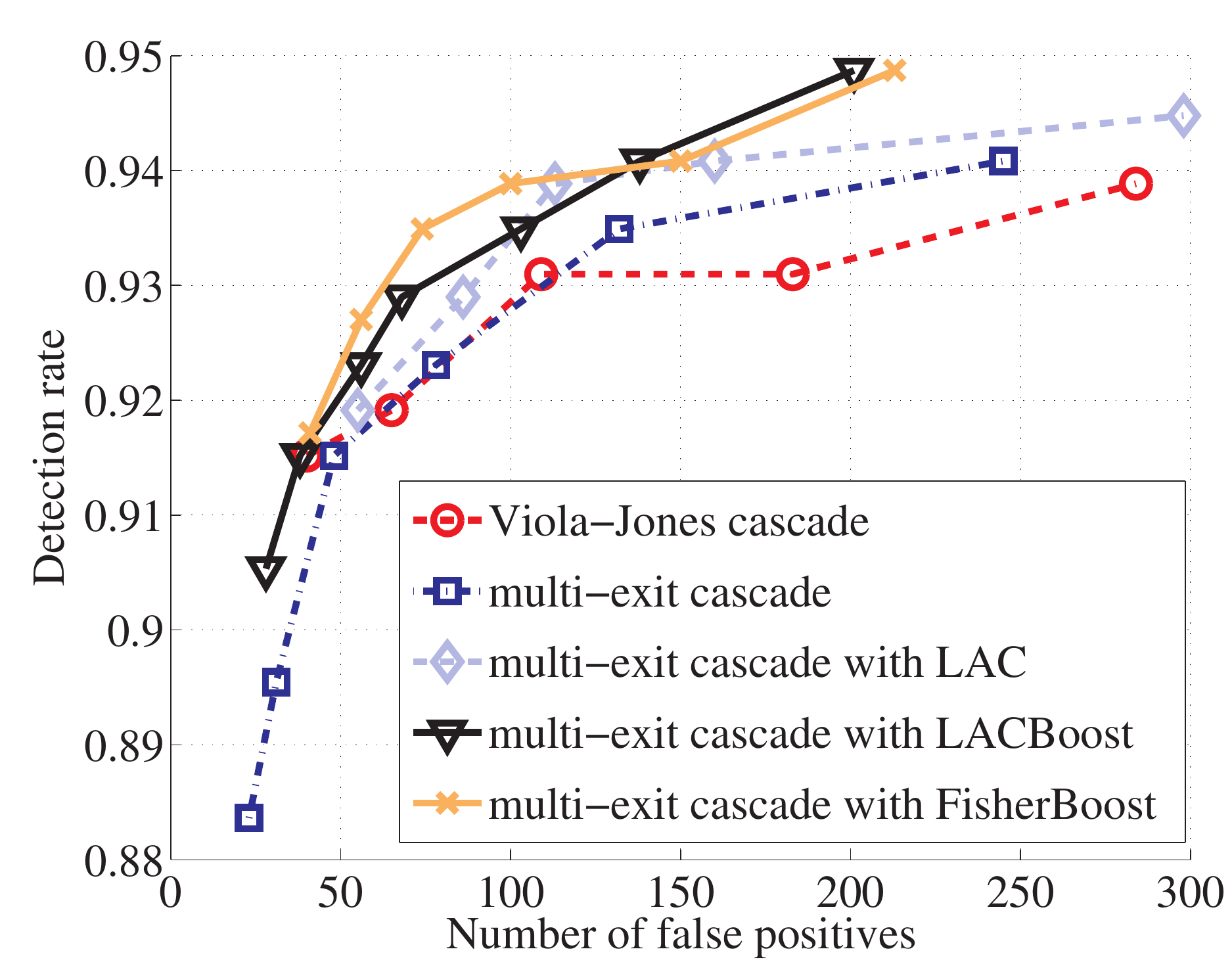}
        \includegraphics[width=.38\textwidth]{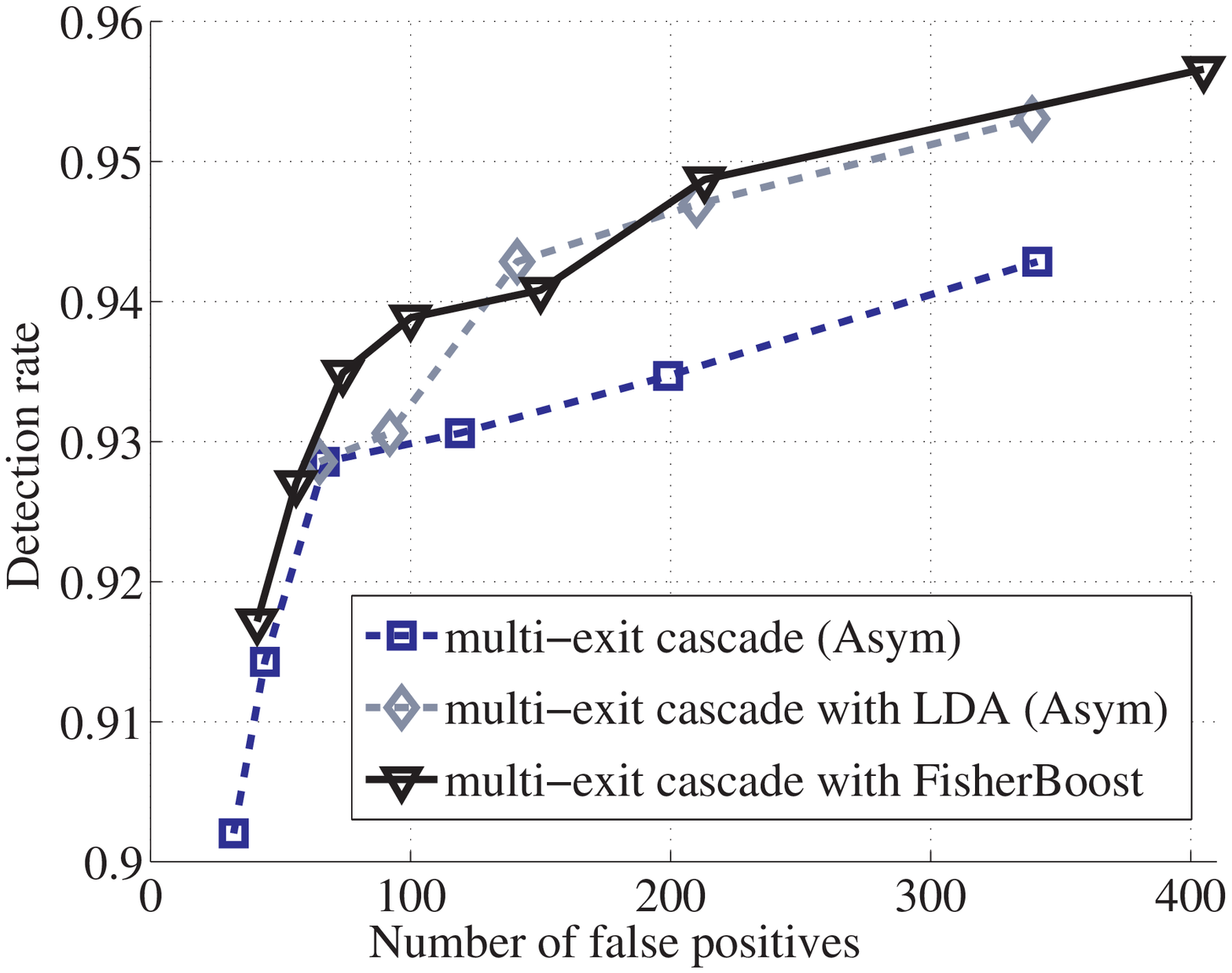}
        \includegraphics[width=.38\textwidth]{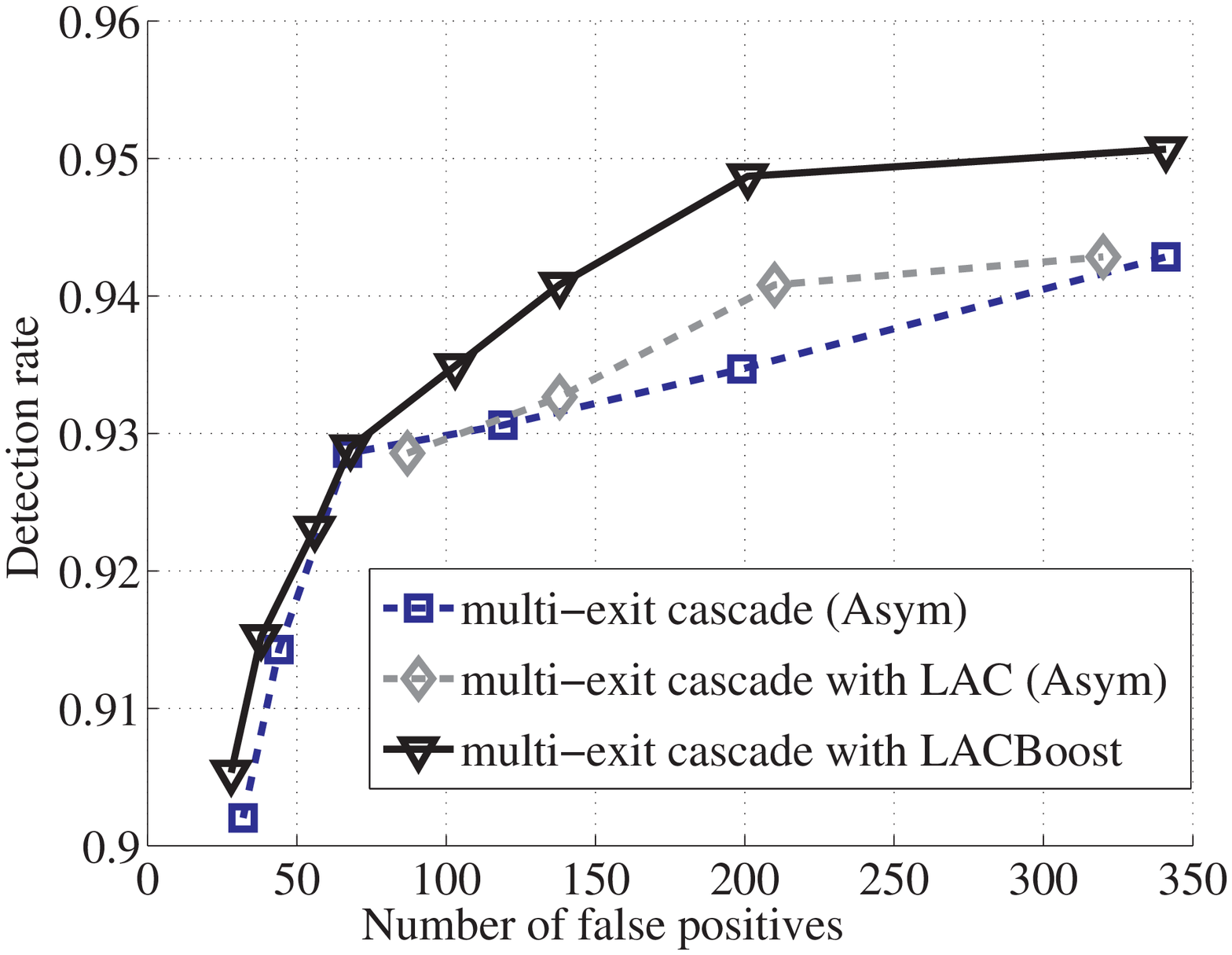}
    \end{center}
    \caption{Cascade performances using ROC curves 
    (number of false positives versus detection rate) on the MIT+CMU test data.
    ``Ada'' means that features are selected using AdaBoost. Viola-Jones cascade
    is the method in
    \cite{viola2004robust}.
      ``Asym'' means that features are selected using AsymBoost.
    }
    \label{fig:ROC1}
\end{figure*}

%
%
%
%
%

\comment{
\begin{figure}[t]
    \begin{center}
        \includegraphics[width=.45\textwidth]{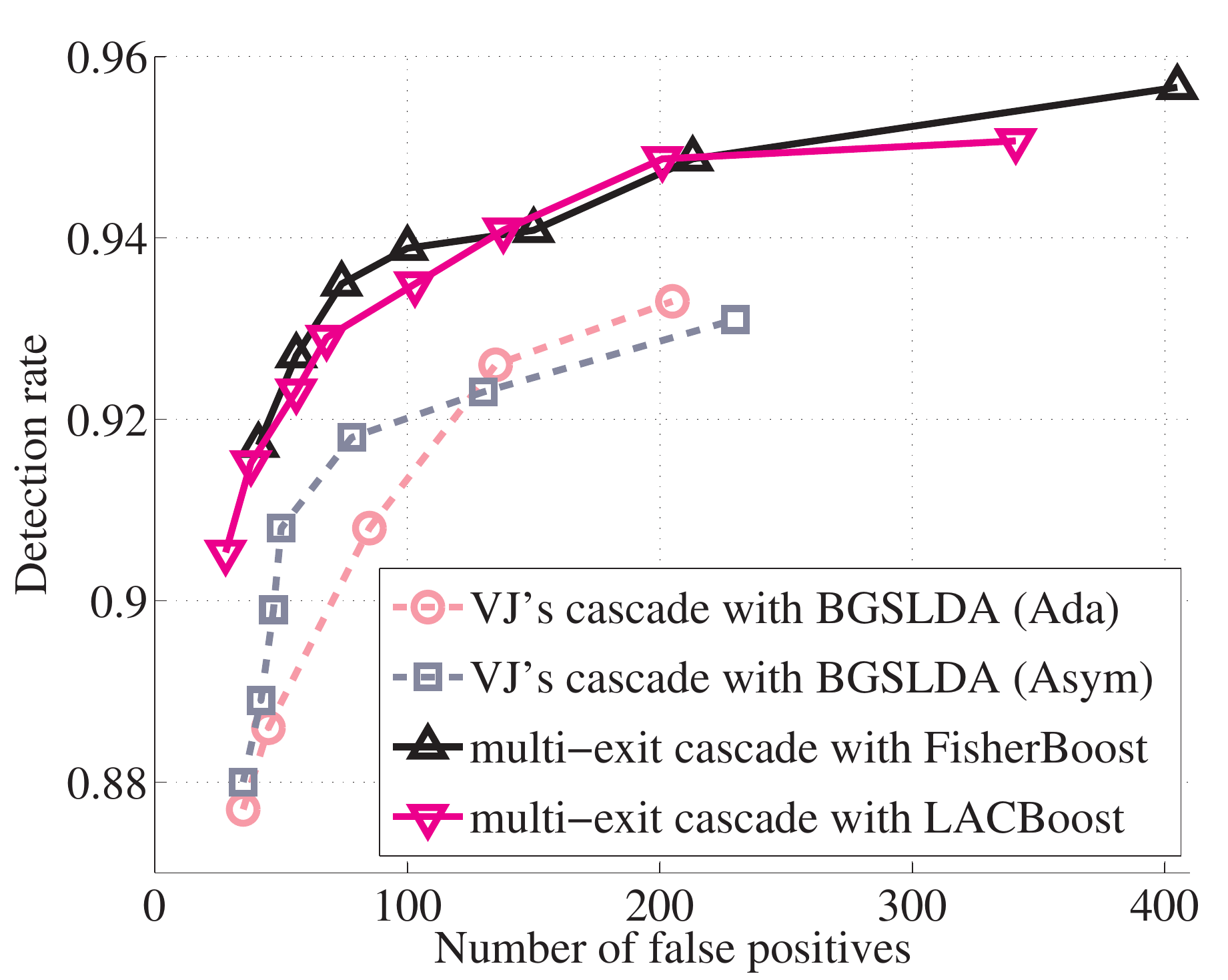}
    \end{center}
    \caption{Cascade performances on the MIT+CMU test data. We compare our methods with BGSLDA in
    \cite{Paisitkriangkrai2009CVPR}.}
    \label{fig:ROC_BGSLDA}
\end{figure}
}

\subsection{Pedestrian Detection Using a Cascade Classifier}

\begin{figure*}[t]
    \begin{center}
        \includegraphics[width=.2438\textwidth]{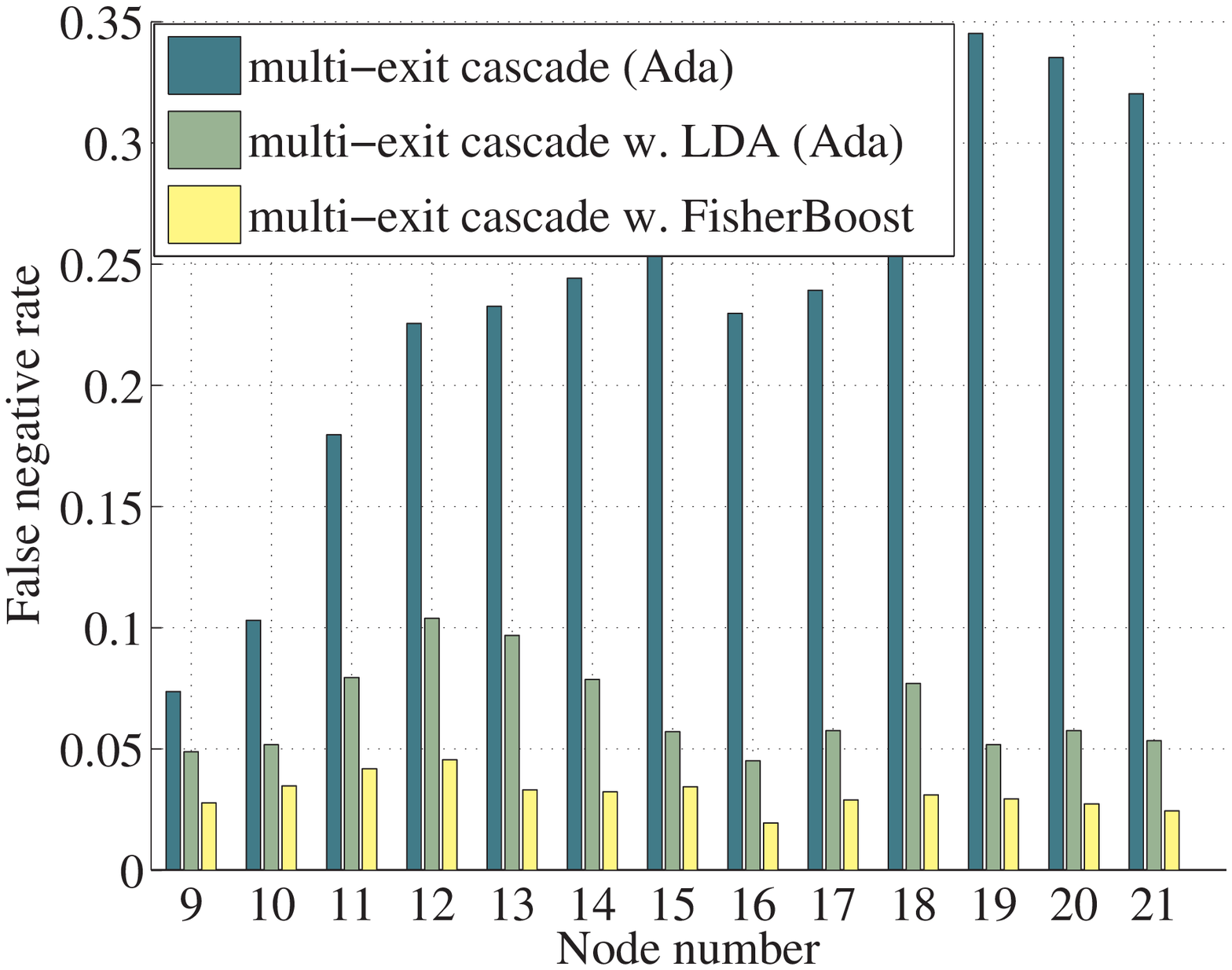}
        \includegraphics[width=.2438\textwidth]{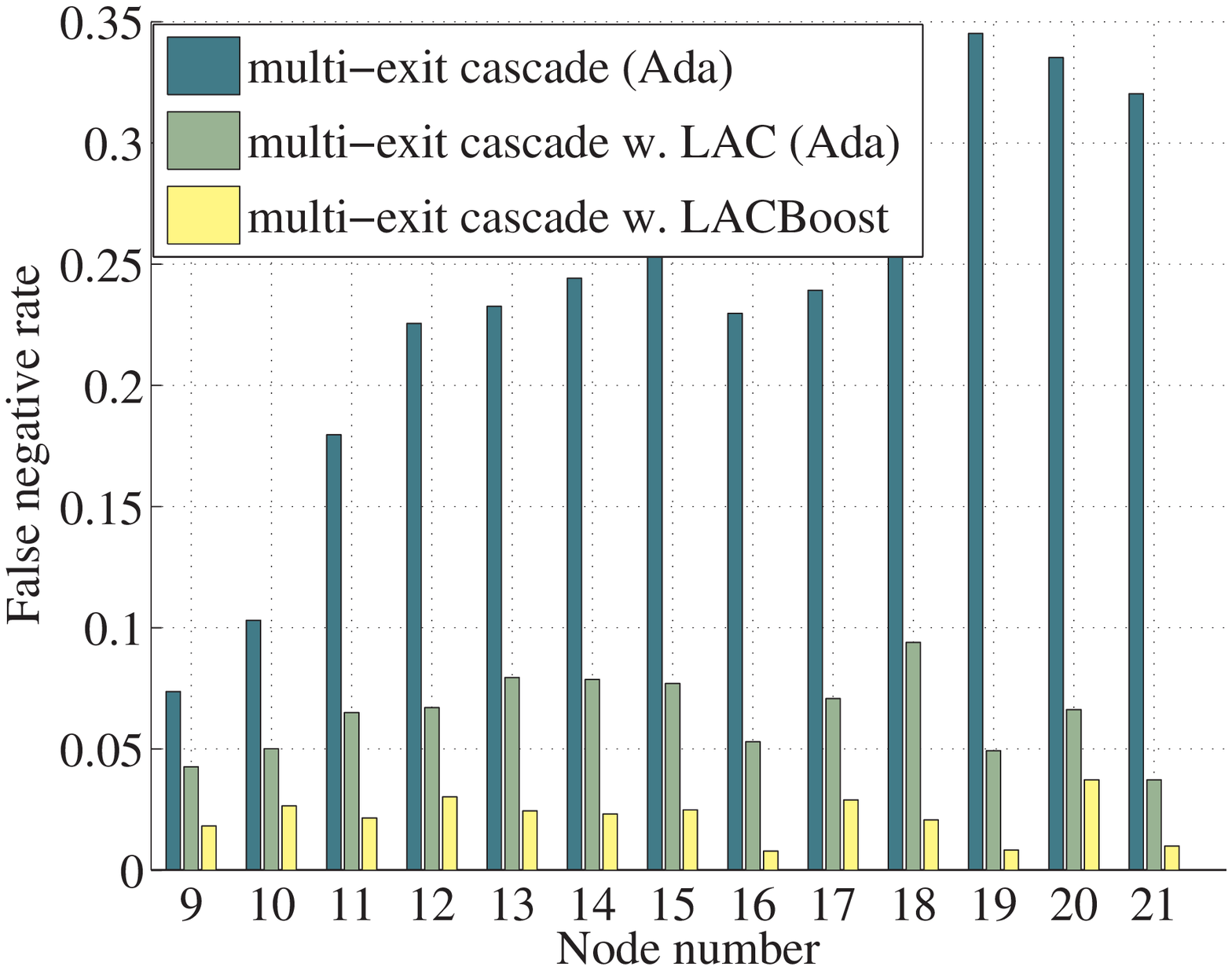}
        \includegraphics[width=.2438\textwidth]{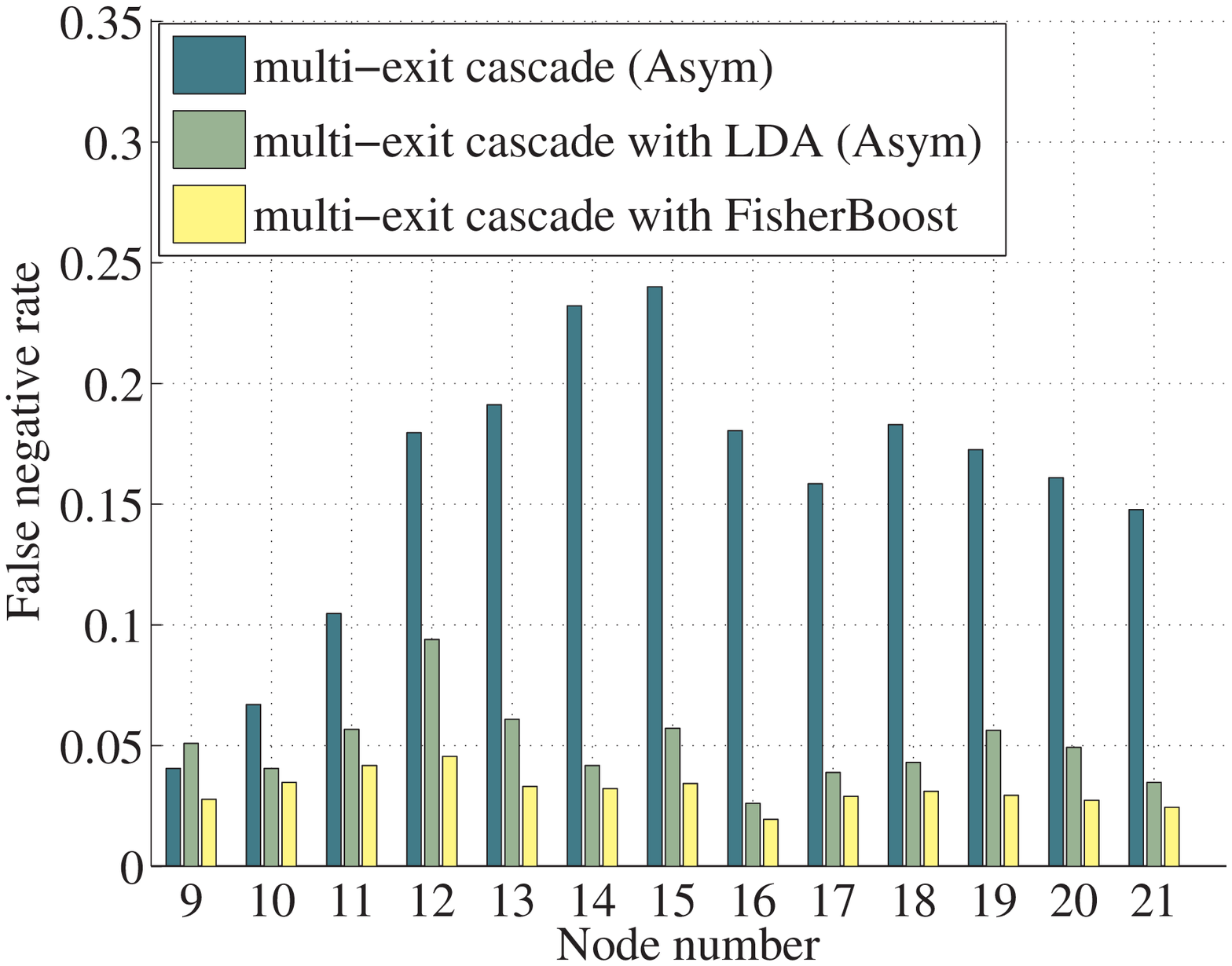}
        \includegraphics[width=.2438\textwidth]{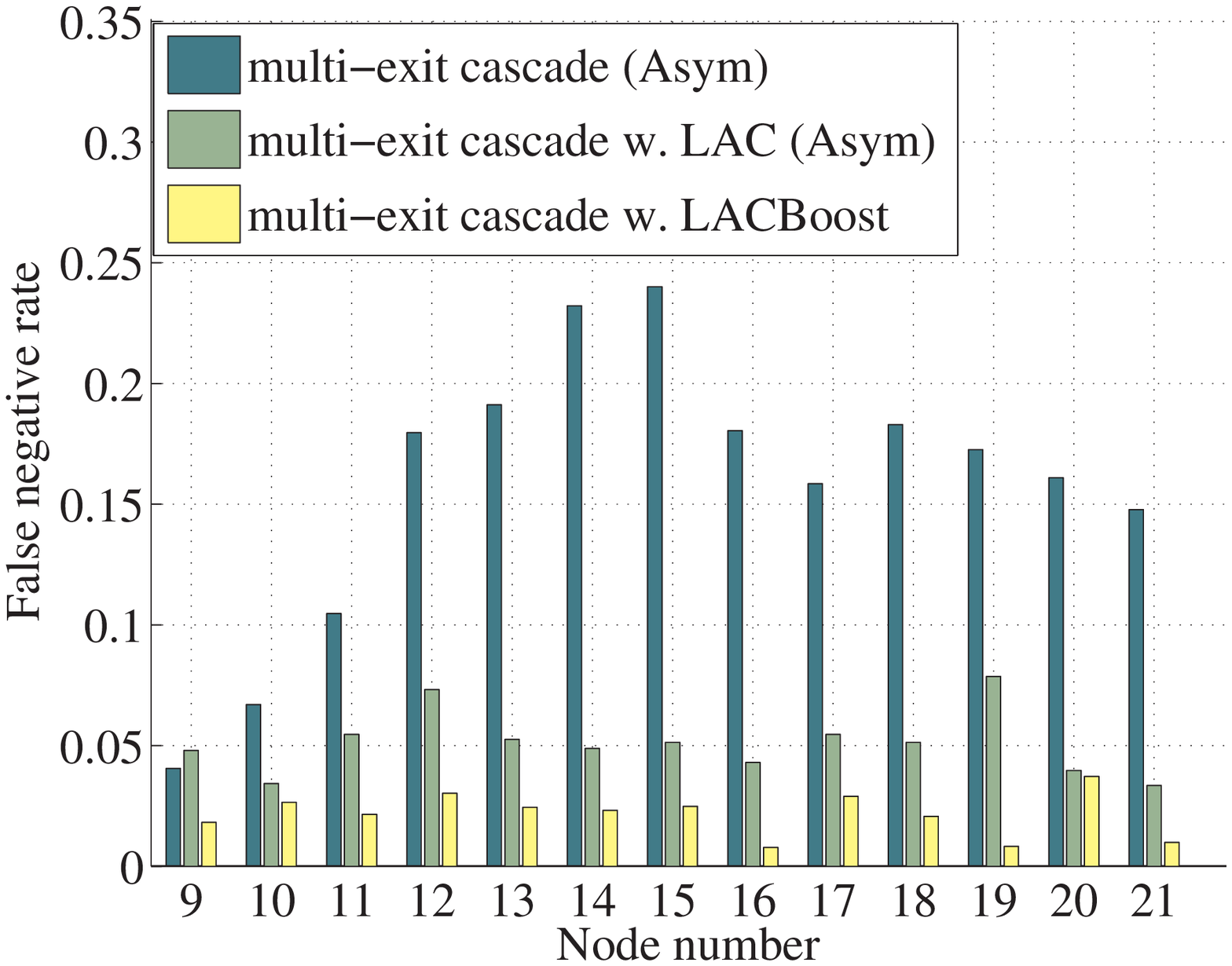}
    \end{center}
    \caption{Node performances on the validation data (INRIA pedestrian detection). 
       ``Ada'' means that features are selected using AdaBoost;
    ``Asym'' means that features are selected using AsymBoost.
    }
    \label{fig:node_ped}
\end{figure*}

\begin{figure*}[t!]
    \begin{center}
        \includegraphics[width=.38\textwidth]{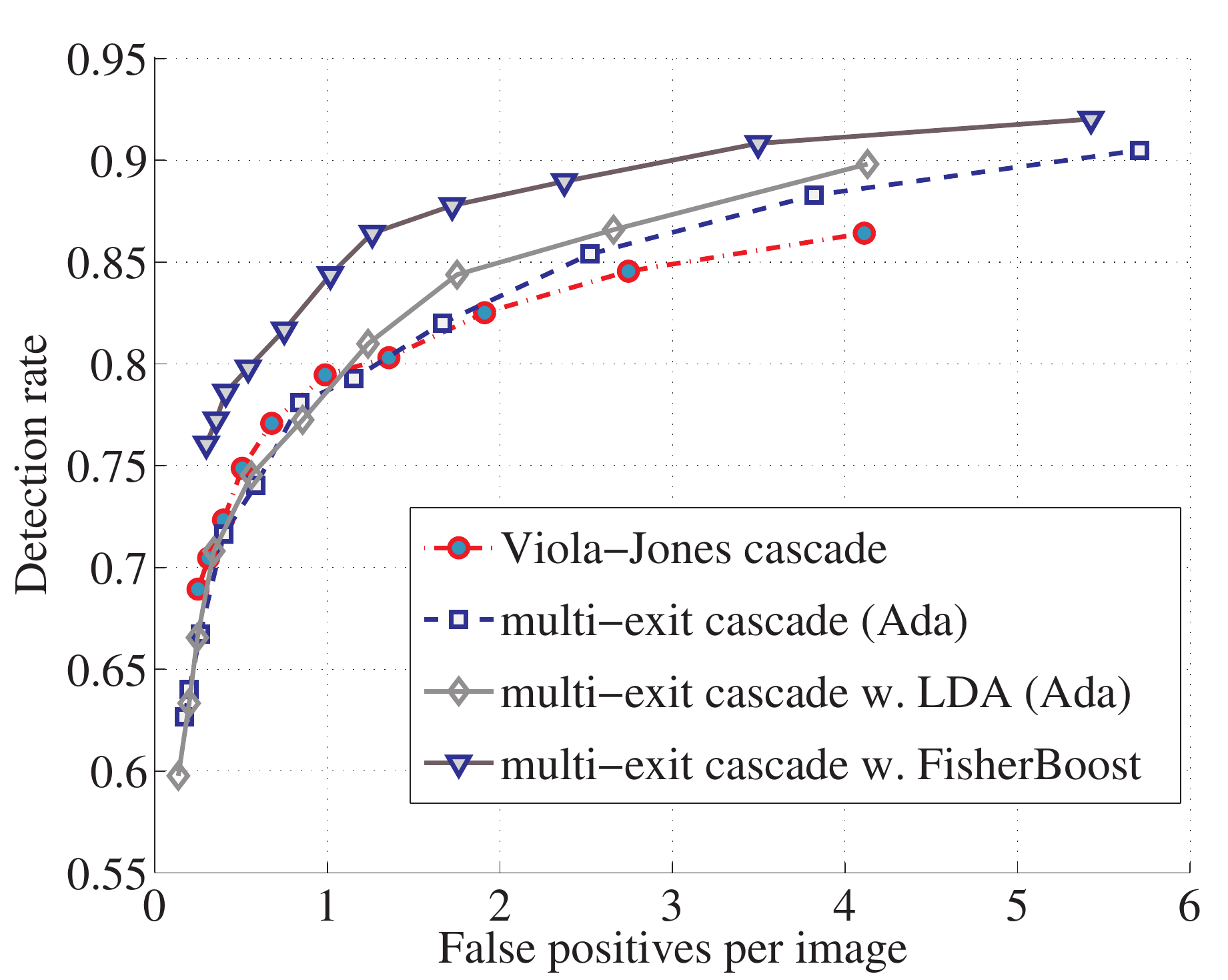}
        \includegraphics[width=.38\textwidth]{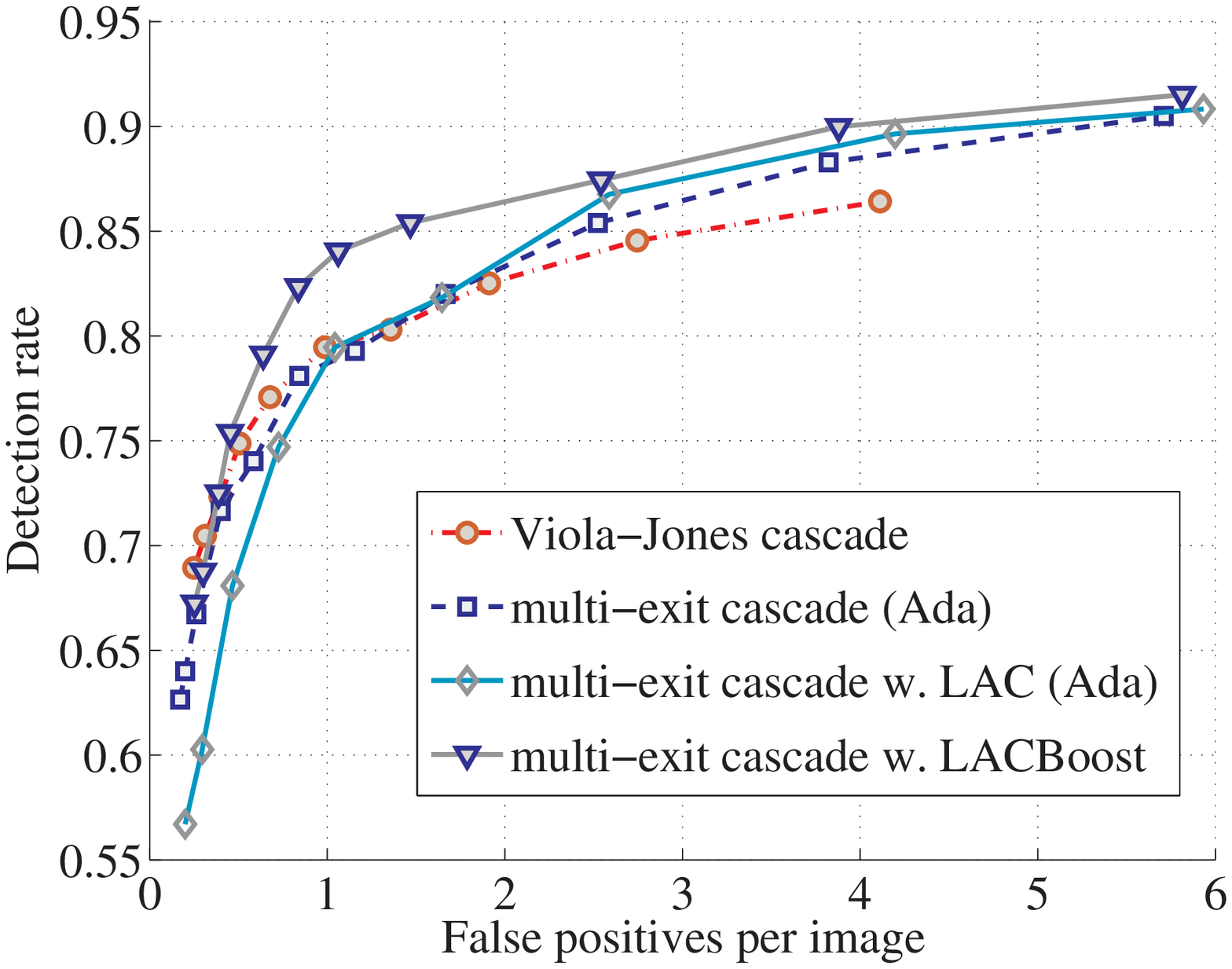}
        \includegraphics[width=.38\textwidth]{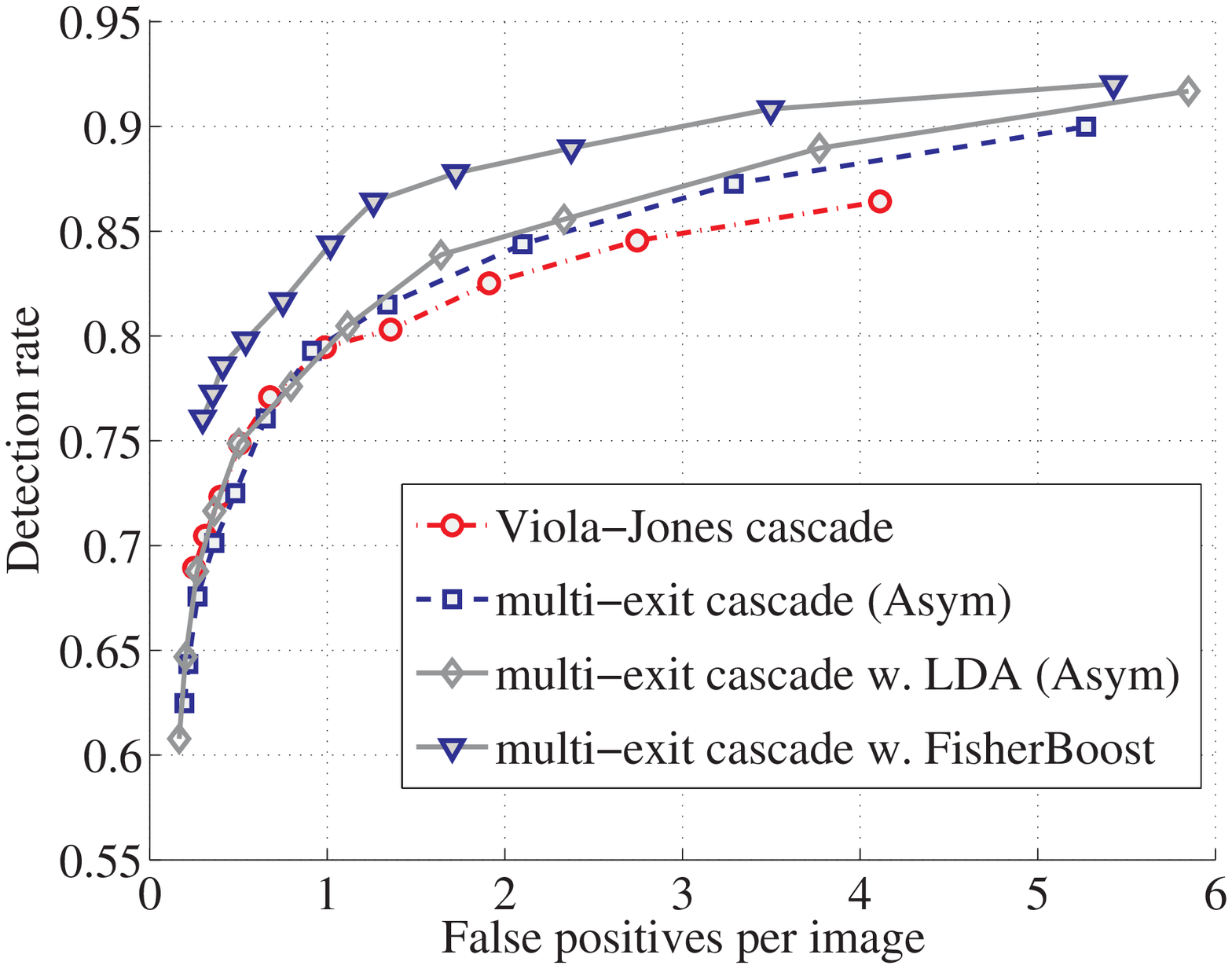}
        \includegraphics[width=.38\textwidth]{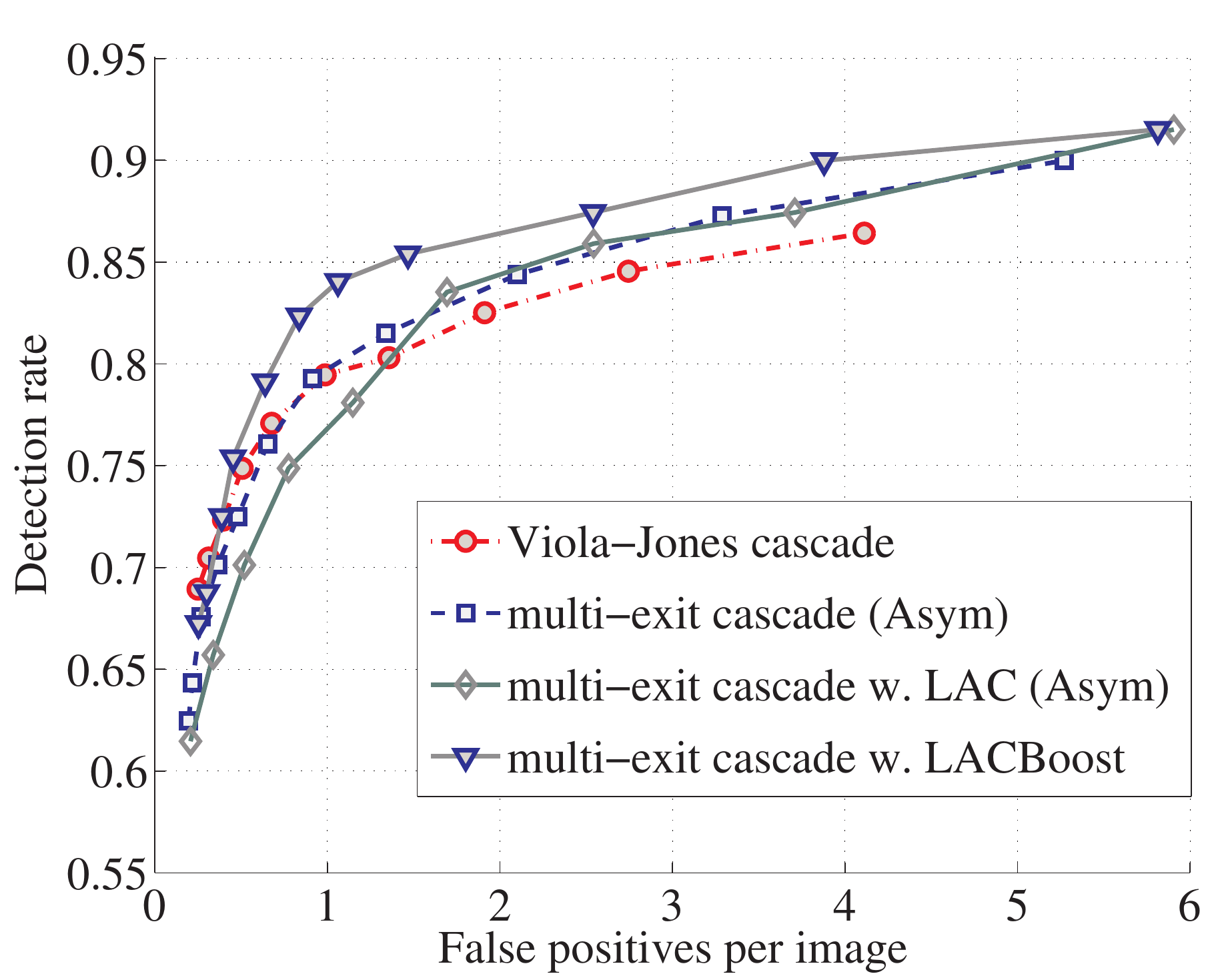}
    \end{center}
    \caption{Cascade performances in ROC  
    (false positives per image versus detection rate) on the INRIA pedestrian 
    data.
    ``Ada'' means that features are selected using AdaBoost. Viola-Jones cascade
    is the method in
    \cite{viola2004robust} with weighted LDA on HOG as weak classifiers.
      ``Asym'' means that features are selected using AsymBoost.
    }
    \label{fig:ROC_Ped}
\end{figure*}

\input{INRIA_Exp.tex}

\section{Conclusion}
\label{sec:con}

    By explicitly taking into account the node learning goal in cascade classifiers, we have
    designed new boosting algorithms for more effective object detection.  Experiments validate the
    superiority of the methods developed, which we have labeled FisherBoost and LACBoost.  We have
    also proposed the use of entropic gradient descent to efficiently implement FisherBoost and
    LACBoost. The proposed algorithms are easy to implement and can be applied to other asymmetric
    classification tasks in computer vision.  We aim in future to design new asymmetric boosting
    algorithms by exploiting asymmetric kernel classification methods such as \cite{Tu2010Cost}.  
    Compared with stage-wise AdaBoost, which is parameter-free, our boosting algorithms need to tune 
    a parameter.  
    We are also interested in developing parameter-free 
    stage-wise boosting that considers the node learning
    objective.

\bibliographystyle{ieee}
\bibliography{lac}

\input{bio.tex}

\end{document}

%% file: alg_multiexit_cascade.tex
\def\dmin{ d_{\mathrm{min}} }
\def\fmax{ f_{\mathrm{max}} }
\def\ftarget{ F_{\mathrm{fp}} }
\setcounter{AlgoLine}{0}
\linesnumbered\SetVline
\begin{algorithm}[t]
\caption{The procedure for training a multi-exit cascade with LACBoost or FisherBoost.}
\centering   
{\small{
%
   \begin{minipage}[]{0.94\linewidth}
    \KwIn{
  
        \ADot
         A training set with $m$ examples, which are ordered by their labels 
         ($m_1$ positive examples followed by $m_2$ negative examples);

        \ADot
         $\dmin$:  minimum acceptable detection rate per node;
    
         \ADot
         $\fmax$:  maximum acceptable false positive rate per node;
      
        \ADot
         $ \ftarget $: target overall false positive rate.
         }   
   { {\bf Initialize}:
   \\
   $t = 0$;    ({\em node index}) \\
   $ n = 0 $;  ({\em total selected weak classifiers up to the current node})  
   \\
   $D_t = 1$;  
   $ F_t = 1$. ({\em overall detection rate and false positive 
                     rate up to the current node})
   }

\While{ $  \ftarget < F_t $ }
{
$ t = t + 1$;    ({\em increment node  index})
\\
      \While{
              $ d_t < \dmin  $ 
            }
      { 
       ({\em current detection rate $ d_t $ is not acceptable yet})\\
    \ADot 
        $ n = n + 1$, and generate a weak classifier and update all the weak classifiers'
        linear coefficient 
        using LACBoost or FisherBoost.
       %
       %
       
     \ADot 
       Adjust threshold $b$ of the current boosted strong classifier 
       \[ F^t({\bf x}) =
       \sum_{j = 1}^{n} w^t_j  h_j ({\bf x}) - b\]
       such that $f_t \approx  \fmax$. 
     
     \ADot
        Update the detection rate of the current node 
       $d_t$ with the learned boosted classifier. 
      }
      Update $ D_{t+1} = D_t \times d_t $; $ F_{t+1} = F_t \times f_t $\;

      Remove correctly classified negative samples from negative training set.
      
      \If{ $\ftarget < F_t $ }
      {Evaluate the current cascaded classifier on the negative images
      and add misclassified samples into the negative training set;
      ({\em bootstrap})
      }
}
\KwOut{
      A multi-exit cascade classifier with $n$ weak classifiers
      and $t $ nodes.

}

\end{minipage}
}}
    \label{alg:MultiExit_LAC}
\end{algorithm}

%% file: INRIA_Exp.tex
    In this experiment, we use the INRIA pedestrian data \cite{Dalal2005HOG} 
    to compare the performance of our algorithms with other
    state-of-the-art methods.  There are $2,416$ cropped mirrored pedestrian images and $1,200$ large
    background images in the training set. 
    The test set contains $288$ images containing $588$ annotated pedestrians and $453$
    non-pedestrian images.

    Each training sample is scaled to $64 \times 128$ pixels with $16$ 
    pixels additional borders for preserving
    the contour information.  
    During testing, the detection scanning window is resized to $32 \times 96$ pixels
    to fit the human body.  
    We have used the histogram of oriented gradient (HOG) features in our experiments.
    Instead of using fixed-size blocks ($105$ blocks of size $16 \times 16$ pixels) as in 
    Dalal and Triggs \cite{Dalal2005HOG},
    we define blocks with different scales (minimum $12 \times$ 12, and maximum $64 \times 128$) 
    and width-length ratios ($1:1, 1:2, 2:1, 1:3$, and $3:1$).  
    Each block is divided into $2 \times 2$ cells, and the HOG 
    in each cell are summarized into $9$ bins.  
    Thus, totally $36$-dimensional features are generated for each block.
    There are in total $ 7,735 $ blocks for a $64 \times 128$-pixel image. 
    $ \ell_1$-norm normalization is then applied to the feature vector. 
    Furthermore, we use integral histograms to speed up the computation as in \cite{zhu2006fast}. 
    At each iteration, we randomly sample $10\%$ of the whole possible blocks for training a weak
    classifier.
    We have used weighted linear discriminant analysis (WLDA) as weak classifiers, same as in
    \cite{paisitkriangkrai2008fast}. Zhu \etal used linear support vector machines as weak
    classifiers \cite{zhu2006fast}, which can also be used as weak classifiers here.

    For all the approaches evaluated, we use the same cascade structure with $21$ nodes and totally 
    $ 612 $ weak classifiers 
    (the first three nodes have four weak classifiers for each, 
    and the last six have $60$ weak classifiers).

    The positive examples are from the INRIA training set and remain the same for each node. 
    The negative examples are obtained by collecting the false positives of currently learned 
    cascade from the large background images with bootstrapping. 
    The parameter $\theta$ of our FisherBoost and LACBoost is selected from   
     $\{
      \frac{1}{10}, 
      \frac{1}{12},
      \frac{1}{14},
      $
      $
      \frac{1}{16},
      \frac{1}{18},
      \frac{1}{20}
     \}$.
     We have not carefully selected $ \theta $ in this experiment. 
     Ideally, cross-validation should be used to pick the best value of $ \theta $ by using
     an independent cross-validation data set. 
     Here because INRIA data set does not have many labeled positive data,
     we have used the same $2,416$ training positives, plus 
     $500$ additional negative examples obtained by bootstrapping for validation.  
     Improvement might be obtained if a large cross-validation data set was available.

    The scale ratio of input image pyramid is $1.09$ and the scanning step-size is $8$ pixels.
    The overlapped detection windows are merged using the simple heuristic strategy proposed 
    by Viola and Jones \cite{viola2004robust}.
    It takes about $5$ hours to train the entire cascade pedestrian detector on the workstation.

    For the same reason described in the face detection section,
    the FisherBoost/LACBoost and Wu \etal's LDA/LAC post-processing
    are applied to the cascade from about the $3$th node, instead of the first node.

    Since the number of weak classifiers of our pedestrian detector is small, 
    we use the original matrix $\Q$ rather than the approximate diagonal matrix in this experiment.

    The Pascal VOC detection  Challenge criterion  
    \cite{pascal2010,paisitkriangkrai2008fast} is adopted here. A detection result is 
    considered true or
    false positive based on the area of overlap with the ground truth
    bounding box. To be considered a correct detection, the area
    of overlap between the predicted bounding box and ground
    truth bounding box must exceed $40\%$ of the union of the prediction and the ground truth. 
    We use the false positives per image (FPPI) metric as suggested in \cite{dollar2009}.

    Fig.~\ref{fig:node_ped} shows the node performances of various configurations on the INRIA
    pedestrian data set. Similar results are obtained as in the face detection experiment: again, 
    our new boosting algorithms significantly outperform 
    AdaBoost \cite{viola2004robust}, AsymBoost \cite{Viola2002Fast},
    and are considerably better than  Wu \etal's post-processing methods \cite{wu2008fast} 
    at most nodes.  
    Compared with the face detection experiment, we obtain more obvious improvement on
    detecting pedestrians. This may be because pedestrian detection is much more 
    difficult and there is more room for improving the detection performance.

    We have also compared the ROC curves of complete cascades, which are plotted in
    Fig. \ref{fig:ROC_Ped}.
    FisherBoost and LACBoost perform better than all other compared methods. 
    In contrast to the results of the detection rate for each node, 
    LACBoost is slightly worse than FisherBoost in some cases 
    (also see Fig. \ref{fig:HOGSVM}).
    In general, LAC and LDA post-processing improve those without post-processing.
    Also we can see that LAC post-processing performs slightly worse than
    other methods at the low false positive part. 
    Probably LAC post-processing over-fits the training data
    in this case. 
    Also, in the same condition, FisherBoost/LDA post-processing 
    seems to perform better than LACBoost/LAC post-processing.  We will discuss this issue in the
    next section.

\begin{figure}[h!]
    \begin{center}
        \includegraphics[width=.38\textwidth]{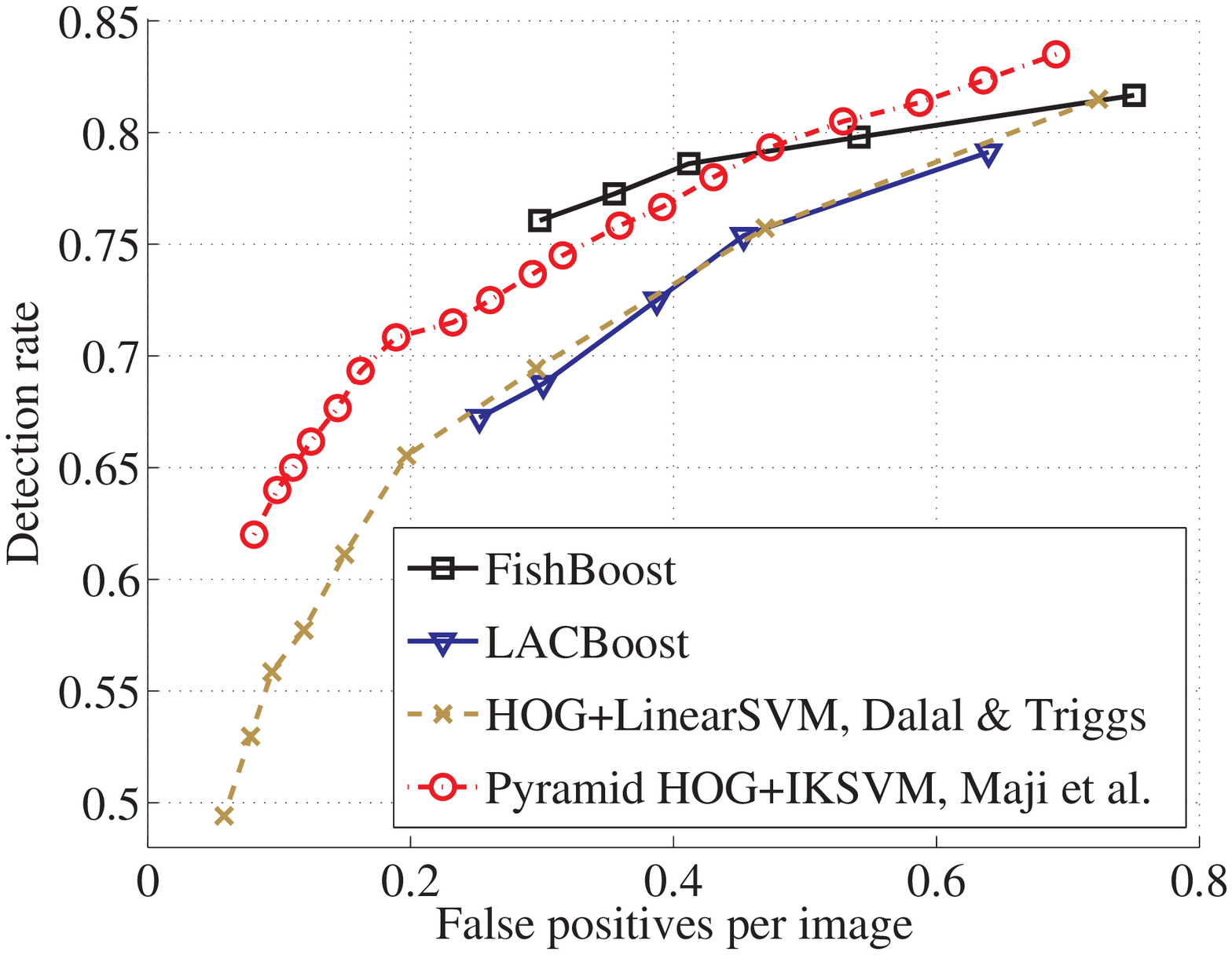}
    \end{center}
    \caption{Comparison between FisherBoost, LACBoost, 
    HOG with the linear SVM of Dalal and Triggs \cite{Dalal2005HOG},
    and Pyramid HOG with the histogram intersection kernel SVM (IKSVM)
    of Maji \etal \cite{maji2008}.
    In detection rate, 
    our FisherBoost improves Dalal and Triggs' approach by over $7\%$ 
    at FPPI of $0.3$ on INRIA's pedestrian detection dataset.}
    \label{fig:HOGSVM}
\end{figure}

    In summary, FisherBoost or LACBoost has superior performance than all the other algorithms.
    We have also compared FisherBoost and LACBoost with HOG with linear SVM of 
    Dalal and Triggs \cite{Dalal2005HOG}, and the state-of-the-art on pedestrian detection---the pyramid HOG (PHOG) with
    histogram intersection kernel SVM (IKSVM) \cite{maji2008}.
%
    We keep all experiment configurations the same, except that 
    HOG with linear SVM and PHOG with IKSVM have employed 
    sophisticated mean shift to merge overlapped detection windows,
    while ours use the simple heuristic of Viola and Jones \cite{viola2004robust}.
    The results are reported in Fig. \ref{fig:HOGSVM} and 
    the observations are: 
 \begin{enumerate}
 \item
     LACBoost performs similarly to Dalal and Triggs' \cite{Dalal2005HOG};
 \item
     PHOG with nonlinear IKSVM performs much better than Dalal and Triggs' HOG with linear SVM.
    This is consistent with the results reported in \cite{maji2008};
 \item
    FisherBoost performs better than PHOG with IKSVM at the low FPPI part (lower than $ 0.5 $).  
 \end{enumerate}
    Note that our  FisherBoost and LACBoost use HOG, instead of PHOG. 
    It is not clear how much gain PHOG has contributed to the final detection performance in
    the case of PHOG plus IKSVM of \cite{maji2008}.\footnote{On object categorization, PHOG seems to
    be a better descriptor than HOG \cite{bosch07}. 
    It is likely that our detectors may perform better if we replace HOG with PHOG.
    We leave this as future work.
    }

    In terms of efficiency in the test phase for each method, 
    our FisherBoost or LACBoost needs about $ 0.7 $ seconds on average on the INRIA test data 
    (no image re-scaling is applied and single CPU core is used on our workstation). 
    PHOG with IKSVM needs about $ 8.3 $ seconds on average. So our boosting framework 
    is about $ 14 $ times faster than PHOG with IKSVM.  
    Note that HOG with linear SVM is much slower ($ 50 $ to $ 70 $ times slower than
    the boosting framework), which agrees  with the results in \cite{zhu2006fast}.

    In both face and pedestrian detection experiments, we have observed that
    FisherBoost performs slightly better than LACBoost. 
    We try to elaborate this on the following.

\subsection{Why LDA Works Better Than LAC}

    Wu \etal observed that in many cases, 
    LDA post-processing gives better 
    detection rates on MIT+CMU face data than LAC \cite{wu2008fast}. 
    When using the LDA criterion to select Haar features, 
    Shen \etal \cite{GSLDA2010Shen} tried different combinations 
    of the two classes' covariance matrices 
    for calculating the within-class matrix: 
    $ \C_w =  \bSigma_1 + \gamma \bSigma_2  $ with $ \gamma $ a nonnegative constant.
    It is easy to see that $ \gamma = 1 $ and $ \gamma = 0 $ 
    correspond to LDA and  LAC, respectively. 
    They found that setting $ \gamma \in  [0.5, 1] $
    gives best results on the MIT+CMU face detection task
    \cite{GSLDA2010Shen,Paisitkriangkrai2009CVPR}.

    According to the analysis in this work, LAC is optimal 
    if the distribution of $[h_1( \x ), h_2(\x), \cdots, h_n(\x) ]$ on 
    the negative data  is symmetric. In practice, this requirement may not be
    perfectly satisfied, especially for the first several node classifiers.
    At the same time, these early nodes have much more impact 
    on the final cascade's detection performance than other nodes. 
    This may explain why in some cases the improvement of LAC is not
    significant. 
    However, this does not explain why LDA (FisherBoost) works;
    and sometimes it is even better
    than LAC (LACBoost). At the first glance,
    LDA (or FisherBoost) by no means  explicitly considers the 
    imbalanced node learning objective.  
    Wu \etal did not have a plausible explanation either 
    \cite{wu2008fast,wu2005linear}. 
    
    \begin{proposition}
        For object detection problems, 
        the Fisher linear discriminant analysis can be viewed as a regularized version
        of linear asymmetric classification. In other words, 
        linear discriminant analysis has already considered
        the asymmetric learning objective. 
    \end{proposition}
    \begin{IEEEproof}
        For object detection such as face and pedestrian detection considered here,
        the covariance matrix of the negative class is close to a scaled 
        identity matrix. 
        In theory, the negative data can be anything other than
        the target. Let us look at one of the off-diagonal elements
        \begin{align}
        \bSigma_{ij, i\neq j}  
        &=  {\mathbb E} [ ( h_i( \x ) - {\mathbb E}[ h_i( \x ) ]  ) 
                         ( h_j( \x ) - {\mathbb E}[ h_j( \x ) ]  ) 
                       ]
        \notag 
        \\
        &= {\mathbb E} [  h_i ( \x )  h_j( \x ) ] \approx 0.
        \label{EQ:900}
        \end{align}
        Here $ \x $ is the image feature of
        the negative class. 
        We can assume that $ \x $ is i.i.d. and 
        approximately, $ \x $ follows a uniform distribution.
        So $ {\mathbb E} [ h_{i,j} ( \x )]   = 0 $.
        That is to say, on the negative class,
        the chance of $ h_{i,j} ( \x ) = +1 $
        or $ h_{i,j} ( \x ) = -1 $ is the same, which is $ 50\% $. 
        Note that this does not apply to the positive class
        because $ \x $ of the positive class is not 
        uniformly distributed.
        The last equality of \eqref{EQ:900} 
        uses the fact that weak classifiers 
        $ h_i( \cdot)  $ and $ h_j( \cdot)  $
        are approximately statistically independent.  
        Although this assumption may not hold in practice
        as pointed out in \cite{shen2010dual},
        it could be a plausible approximation.

        Therefore, the off-diagonal elements of  $ \bSigma $
        are almost all zeros; and $\bSigma$ is a diagonal matrix.
        Moreover in object detection, it is a reasonable assumption
        that the diagonal elements 
        $ {\mathbb E} [  h_j ( \x )  h_j( \x ) ] $
        $ (j=1,2,\cdots ) $ have similar values. 
        Hence, $ \bSigma_2 \approx v \bf I$ holds, 
        with $ v $ a positive constant.
        
        So for object detection,
        the only difference between LAC and LDA is that, 
        for LAC $ \C_w = \frac{m_1}{m} \bSigma_1 $
        and for LDA $ \C_w = \frac{m_1}{m}
                      \bSigma_1 
                      + v \cdot \frac{m_2}{m} \bf I $.
        This concludes the proof.     
    \end{IEEEproof}
    
    It seems that this regularization term can be the reason why
    the LDA post-processing approach
    and FisherBoost works even better than LAC and LACBoost
    in object detection. 
    However, in practice, the negative data are not
    necessarily uniformly distributed. Particularly,
    in latter nodes, bootstrapping makes negative data
    to be those difficult ones.  
    In this case, 
    it may deteriorate  the performance by completely ignoring
    the negative data's covariance information.

    In FisherBoost, this regularization is equivalent to
    have a $ \ell_2 $ norm regularization   on the primal variable 
    $ \w $, $ \Vert \w \Vert_2^2 $,
    in the objective function of the QP problem
    in Section \ref{sec:LACBoost}. 
    Machine learning algorithms like Ridge regression  use
    $\ell_2$ norm regularization.

%

    Fig.~\ref{fig:cov} shows some empirical evidence 
    that $ \bSigma_2 $ is close to a scaled identity matrix. 
    As we can see, the diagonal elements are much larger than those
    off-diagonal elements (off-diagonal ones are close to zeros).

\begin{figure}[t!]
    \centering
          \includegraphics[width=.35\textwidth]{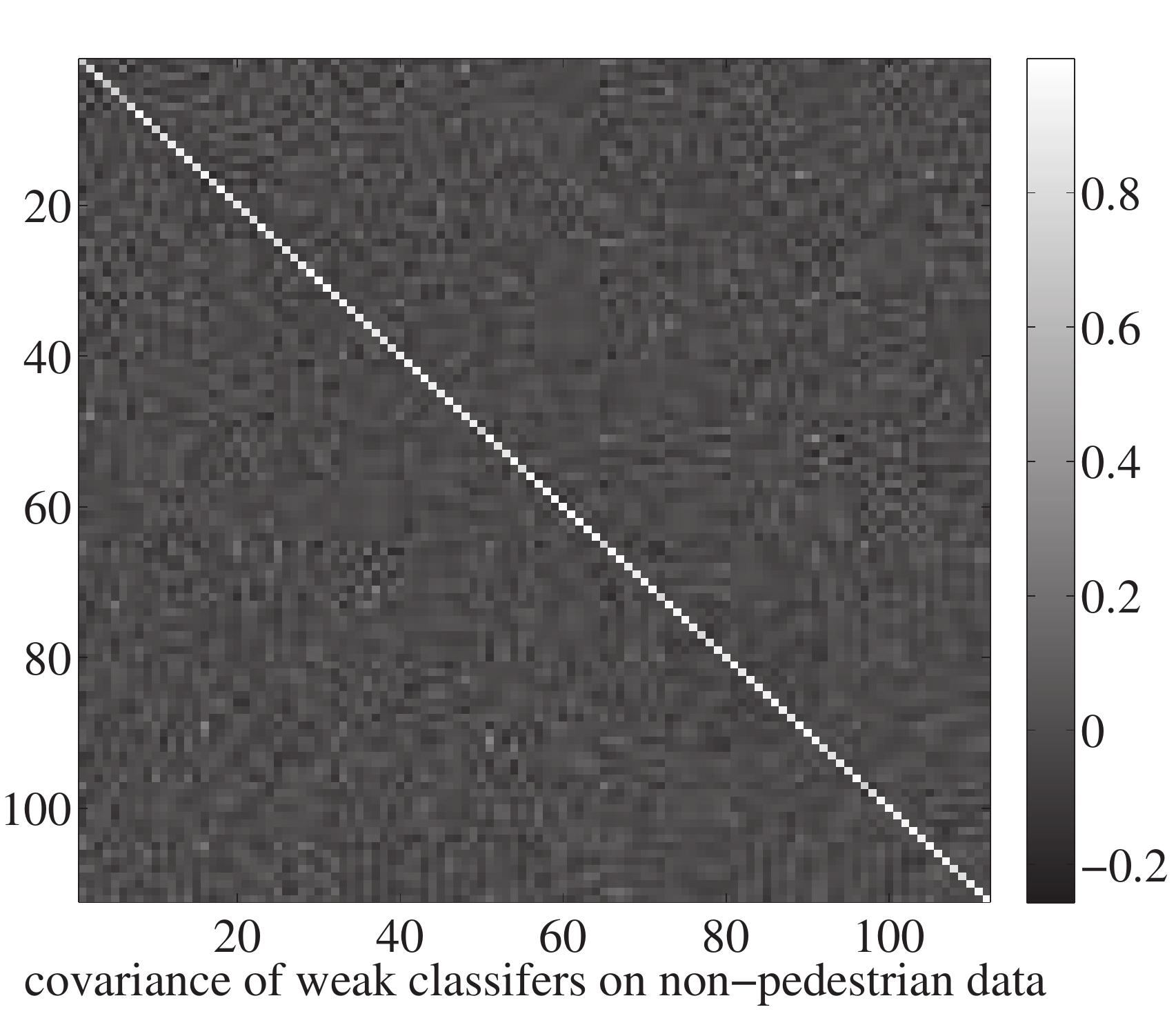}
    \caption{The covariance matrix of the first $112$ weak classifiers selected by
    FisherBoost on non-pedestrian data. It may be approximated by a scaled identity
    matrix. 
    On average, the magnitude of diagonal elements is $20$ times larger than those
    off-diagonal elements.  
    }
    \label{fig:cov}
\end{figure}

%% file: bio.tex


\begin{IEEEbiography}
        [{\includegraphics[width=1in,height=1.25in,clip,keepaspectratio]{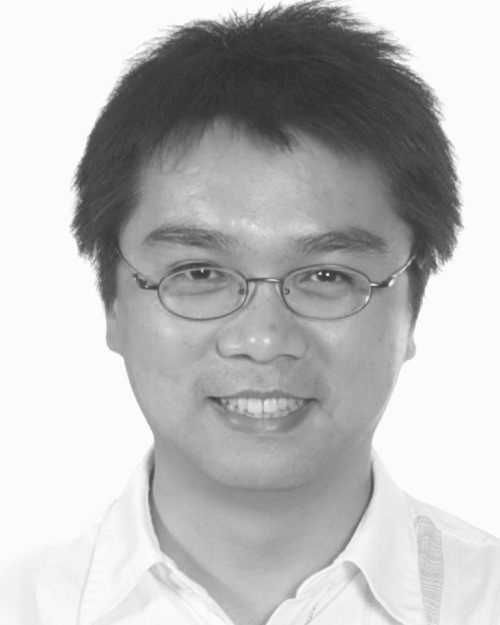}}]
        {Chunhua Shen}
        obtained a Ph.D. degree in Computer Vision from University of Adelaide,
        Australia in 2006, an M.Phil. degree in Applied Statistics from Australian
        National University in 2009, an M.Sc. degree and a Bachelor in 2002 and 1999 
        respectively, both from Nanjing University, Nanjing, China. 
        
        Since Oct. 2005, he has been working with the computer vision program,
        NICTA (National ICT Australia), Canberra Research Laboratory,
        where he is currently a senior researcher. 

        His main research interests include statistical machine learning and 
        computer vision. 
        His recent research focuses on boosting algorithms and their applications in 
        real-time object detection.
        He has published over $50$ peer-reviewed papers in international conferences and journals. 
\end{IEEEbiography}

\vspace{-2em}

\begin{IEEEbiography}
        [{\includegraphics[width=1in,height=1.25in,clip,keepaspectratio]{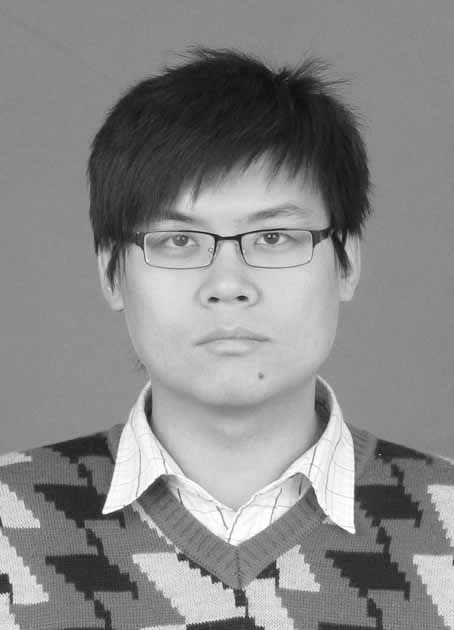}}]
        {Peng Wang} 
        is currently pursuing his  Ph.D. degree at 
        the School of Automation Science and Electrical Engineering,
        Beihang University, Beijing, China.
        He received the B.Sc degree from Beihang University in 2004. 
        His research interests include machine learning and computer vision. 
        From Sept. 2008 to Sept. 2010, he was a visiting scholar at NICTA, Canberra Research
        Laboratory, Canberra, Australia. 
\end{IEEEbiography}

\vspace{-2em}

\begin{IEEEbiography}
     [{\includegraphics[width=1in,height=1.25in,clip,keepaspectratio]{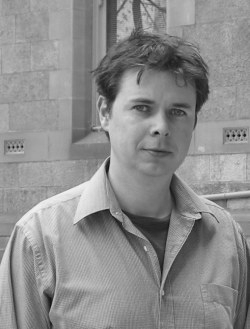}}]
    {Anton van den Hengel} 
        is the founding director of The Australian Center for Visual
        Technologies (ACVT), an interdisciplinary research center focusing on innovation and
        education in the production and analysis of visual digital media.  Prof. van den Hengel
        received a Ph.D. in Computer Vision in 2000, a Masters degree in Computer Science in 1994, a
        Bachelor of Law in 1993, and a Bachelor of Mathematical Science in 1991, all from The
        University of Adelaide.
        Prof. van den Hengel won the CVPR Best Paper Award in 2010. 
\end{IEEEbiography}